\documentclass{article} % For LaTeX2e

\usepackage{iclr2024_conference,times}
\usepackage[utf8]{inputenc} % allow utf-8 input
\usepackage[T1]{fontenc}    % use 8-bit T1 fonts
\usepackage{hyperref}       % hyperlinks
\usepackage{url}            % simple URL typesetting
\usepackage{booktabs}       % professional-quality tables
\usepackage{amsfonts}       % blackboard math symbols
\usepackage{nicefrac}       % compact symbols for 1/2, etc.
\usepackage{microtype}      % microtypography
\usepackage{xcolor}         % colors
\usepackage{natbib}
\usepackage{wrapfig}
\usepackage{caption}
\usepackage{subcaption}

\usepackage{amsmath,amsthm,amsfonts,bm,physics,mathtools}
\usepackage{cancel}
\usepackage{algorithm}
\usepackage{algpseudocode}
\def\1{\bm{1}}

\def\rvepsilon{{\bm{\epsilon}}}

\def\rvf{{\mathbf{f}}}

\def\rvh{{\mathbf{h}}}

\def\rvs{{\mathbf{s}}}

\def\rvw{{\mathbf{w}}}
\def\rvx{{\mathbf{x}}}
\def\rvy{{\mathbf{y}}}

\def\vzero{{\bm{0}}}

\def\vI{{\bm{I}}}

\def\mI{{\bm{I}}}

\DeclareMathAlphabet{\mathsfit}{\encodingdefault}{\sfdefault}{m}{sl}
\SetMathAlphabet{\mathsfit}{bold}{\encodingdefault}{\sfdefault}{bx}{n}

\def\gL{{\mathcal{L}}}

\def\gN{{\mathcal{N}}}
\def\gO{{\mathcal{O}}}

\def\gU{{\mathcal{U}}}

\newcommand{\pdata}{q_{\rm{data}}}
\newcommand{\pprior}{p_{\rm{prior}}}
\newcommand{\E}{\mathbb{E}}
\newcommand{\Ls}{\mathcal{L}}
\newcommand{\R}{\mathbb{R}}

\newcommand{\defeq}{\vcentcolon=}

\newcommand{\parderiv}[2]{\frac{\partial}{\partial #2} #1}
\newcommand{\gradnd}[2]{\nabla_{#2} #1}

\usepackage{array}
\usepackage{multirow}
\usepackage{booktabs}
\usepackage{color, colortbl}
\usepackage{makecell}
\usepackage{float}
\usepackage{thm-restate}
\usepackage{diagbox}
\usepackage{tikz}

\newcommand{\scitenum}[1]{[\citenum{#1}]}
\theoremstyle{plain}

\theoremstyle{definition}

\theoremstyle{remark}

\definecolor{Gray}{gray}{0.85}
\definecolor{LightCyan}{rgb}{0.88,1,1}

\makeatletter
\usepackage{xspace}
\def\@onedot{\ifx\@let@token.\else.\null\fi\xspace}
\DeclareRobustCommand\onedot{\futurelet\@let@token\@onedot}

\newcommand{\figref}[1]{Figure~\ref{#1}}
\newcommand{\equref}[1]{Eq\onedot~\eqref{#1}}
\newcommand{\twoequref}[2]{Eq\onedot~\eqref{#1} and \eqref{#2}}

\newcommand{\tabref}[1]{Table~\ref{#1}}
\newcommand{\thmref}[1]{Theorem~\ref{#1}}

\newcommand{\appref}[1]{Appendix~\ref{#1}}

\newcommand{\algoref}[1]{Algorithm~\ref{#1}}

\newcommand*\circled[1]{\tikz[baseline=(char.base)]{
            \node[shape=circle,draw,inner sep=0.5pt] (char) {#1};}}

\newcommand{\xmark}{\ding{55}}
\newcommand{\cmark}{\ding{51}}

\def\eg{\emph{e.g}\onedot}

\def\ie{\emph{i.e}\onedot}

\def\wrt{w.r.t\onedot}

\newcommand{\mypara}[1]{\noindent\textbf{#1}}

\newcommand{\model}[1][]{DDBM\space}

\usepackage{pifont}
\usepackage{makecell}
\usepackage{graphicx}
\usepackage{tabularx}
\usepackage{multirow}
\usepackage{tablefootnote}
\usepackage{algorithmicx}

\title{Denoising Diffusion Bridge Models }

\author{
Linqi Zhou \quad\quad\quad Aaron Lou \quad\quad\quad Samar Khanna \quad\quad\quad Stefano Ermon \\
Department of Computer Science, Stanford University \\
\texttt{\{linqizhou, aaronlou, samar.khanna, ermon\}@stanford.edu} \\
}
\iclrfinalcopy % Uncomment for camera-ready version, but NOT for submission.
\begin{document}

\maketitle

\begin{abstract}
% Despite great advances in theory of translating between arbitrary distributions via stochastic processes, the methodologies have seldom been met with empirical success that achieve state-of-the-art results. In this work we emphasize that by adopting a reverse time perspective on \textit{diffusion bridges}, a family of stochastic processes that translate between arbitrary distributions, we can unify many existing generative paradigms such as diffusion models~\citep{song2020score, ho2020denoising} and rectified flow~\citep{liu2022flow} under a common framework. We specify a parameterization space under which these prior works are special cases. In doing so, we can adapt many design choices that have made diffusion models successful to methods that focus on translating between arbitrary distributions. We additionally show that for translating between arbitrary distributions, our framework still has tractable marginals to sample from during training, and, during inference, can reuse many diffusion-based sampling methods without modification. Empirically, our method can be easily adapted to both pixel-space and latent-space translation, and acheive \alex{number} on \alex{baselines}.

    Diffusion models are powerful generative models that map noise to data using stochastic processes. However, for many applications such as image editing, the model input comes from a distribution that is not random noise. As such, diffusion models must rely on cumbersome methods like guidance or projected sampling to incorporate this information in the generative process. In our work, we propose Denoising Diffusion Bridge Models (DDBMs), a natural alternative to this paradigm based on \textit{diffusion bridges}, a family of processes that interpolate between two paired distributions given as endpoints. Our method 
    learns the score of the diffusion bridge from data and maps from one endpoint distribution to the other by solving a (stochastic) differential equation based on the learned score. Our method naturally unifies several classes of generative models, such as score-based diffusion models and OT-Flow-Matching, allowing us to adapt existing design and architectural choices to our more general problem. Empirically, we apply DDBMs to challenging image datasets in both pixel and latent space. On standard image translation problems, DDBMs achieve significant improvement over baseline methods, and, when we reduce the problem to image generation by setting the source distribution to random noise, DDBMs achieve comparable FID scores to state-of-the-art methods despite being built for a more general task.
\end{abstract}
\section{Introduction}

Diffusion models are a powerful class of generative models which learn to reverse a diffusion process mapping data to noise \citep{sohl2015deep,song2019generative, ho2020denoising, song2020score}. For image generation tasks, they have surpassed GAN-based methods \citep{Goodfellow2014GenerativeAN} and achieved a new state-of-the-art for perceptual quality \citep{dhariwal2021diffusion}. Furthermore, these capabilities have spurred the development of modern text-to-image generative AI systems\citep{Ramesh2022HierarchicalTI}.

Despite these impressive results, standard diffusion models are ill-suited for other tasks. In particular, the diffusion framework assumes that the prior distribution is random noise, which makes it difficult to adapt to tasks such as image translation, where the goal is to map between pairs of images. As such, one resorts to cumbersome techniques, such as conditioning the model \citep{ho2022classifier, Saharia2021PaletteID} or manually altering the sampling procedure \citep{meng2022sdedit, song2020score}. These methods are not theoretically principled and map in one direction (typically from corrupted to clean images), losing the cycle consistency condition \citep{Zhu2017UnpairedIT}. 

Instead, we consider methods which directly model a transport between two arbitrary probability distributions. This framework naturally captures the desiderata of image translation, but existing methods fall short empirically. For instance, ODE based flow-matching methods \citep{lipman2023flow,albergo2023building,liu2022flow}, which learn a deterministic path between two arbitrary probability distributions, have mainly been applied to image generation problems and have not been investigated for image translation. Furthermore, on image generation, ODE methods have not achieved the same empirical success as diffusion models.
Schr\"odinger Bridge and models \citep{de2021diffusion} are another type of model which instead learn an entropic optimal transport between two probability distributions. However, these rely on expensive iterative approximation methods and have also found limited empirical use. More recent extensions including Diffusion Bridge Matching~\citep{shi2023diffusion,peluchetti2023diffusion} similarly require expensive iterative calculations.

In our work, we seek a scalable alternative that unifies diffusion-based unconditional generation methods and transport-based distribution translation methods, and we name our general framework Denoising Diffusion Bridge Models (DDBMs). We consider a reverse-time perspective of \textit{diffusion bridges}, a diffusion process conditioned on given endpoints, and use this perspective to establish a general framework for distribution translation. We then note that this framework subsumes existing generative modeling paradigms such as score matching diffusion models~\citep{song2020score} and flow matching optimal transport paths~\citep{albergo2023building,lipman2023flow,liu2022flow}. This allows us to reapply many design choices to our more general task. In particular, we use this to generalize and improve the architecture pre-conditioning, noise schedule, and model sampler, minimizing input sensitivity and stabilizing performance. We then apply DDBMs to high-dimensional images using both pixel and latent space based models. For standard image translation tasks, we achieve better image quality (as measured by FID \citep{heusel2017gans}) and significantly better translation faithfulness (as measured by LPIPS~\citep{zhang2018perceptual} and MSE). Furthermore, when we reduce our problem to image generation, we match standard diffusion model performance. 

\section{Preliminaries}
\label{sec:prelim}

Recent advances in generative models have relied on the classical notion of transporting a data distribution $\pdata(\rvx)$ gradually to a prior distribution $\pprior(\rvx)$ \citep{Villani2008OptimalTO}. By learning to reverse this process, one can sample from the prior and generate realistic samples.

\subsection{Generative Modeling with Diffusion Models}

\mypara{Diffusion process.}
We are interested in modeling the distribution $\pdata(\rvx)$, for $\rvx \in \R^d$. We do this by constructing a diffusion process, which is represented by a set of time-indexed variables $\{\rvx_{t}\}_{t=0}^{T}$ such that $\rvx_0 \sim p_0(\rvx) := \pdata(\rvx)$ and $\rvx_T \sim p_T(\rvx) := \pprior(\rvx)$. Here $\pdata(\rvx)$ is the initial ``data" distribution and $\pprior(\rvx)$ is the final ``prior" distribution. The process can be modeled as the solution to the following SDE
\begin{align}\label{eq:diffusion}
\quad d\rvx_t = \rvf(\rvx_t, t) dt + g(t) d\rvw_t
\end{align}
where $\rvf: \R^d\times [0,T] \rightarrow \R^d$ is vector-valued \textit{drift} function, $g: [0,T]\rightarrow \R$ is a scalar-valued \textit{diffusion} coefficient, and $\rvw_t$ is a Wiener process. Following this diffusion process forward in time constrains the final variable $\rvx_T$ to follow distribution $\pprior(\rvx)$.  The reverse of this process is given by
\begin{align}\label{eq:reverse_diffusion}
\quad d\rvx_t = \rvf(\rvx_t, t) - g(t)^2 \nabla_{\rvx_t}\log p(\rvx_t)) dt + g(t) d\rvw_t
\end{align}
where $p(\rvx_t) \defeq p(\rvx_t, t)$ is the marginal distribution of $\rvx_t$ at time $t$. Furthermore, one can derive an equivalent deterministic process called the probability flow ODE \citep{song2020score}, which has the same marginal distributions:
\begin{align}\label{eq:diffusion-ode}
    d\rvx_t = \Big[\rvf(\rvx_t, t) - \frac{1}{2}g(t)^2 \nabla_{\rvx_t}\log p(\rvx_t)\Big] dt
\end{align}
In particular, one can draw $\rvx_T\sim \pdata(\rvy)$ and sample $\pdata$ by solving either the above reverse SDE or ODE backward in time.

\mypara{Denoising score-matching.}
The score, $\gradnd{\log p(\rvx_t)}{\rvx_t}$, can be learned by the score-matching loss
\begin{align}\label{eq:sm}
    \Ls(\theta) = \E_{\rvx_t\sim p(\rvx_t\mid \rvx_0), \rvx_0 \sim \pdata(\rvx),t\sim\gU(0,T)}\Big[\norm{\rvs_\theta(\rvx_t, t) - \gradnd{\log p(\rvx_t\mid\rvx_0)}{\rvx_t}}^2\Big]
\end{align}
such that the minimizer $\rvs^*_\theta(\rvx_t, t)$ of the above loss approximates the true score. Crucially, the above loss is tractable because the transition kernel $p(\rvx_t\mid \rvx_0)$, which depends on specific choices of drift and diffusion functions, is designed to be Gaussian $\rvx_t = \alpha_t \rvx_0 + \sigma_t \rvepsilon$, where $\alpha_t$ and $\sigma_t$ are functions of time and $\rvepsilon\sim \gN(\vzero, \vI)$. It is also common to view the diffusion process in terms of the $\rvx_t$'s signal-to-noise ratio (SNR), defined as $\alpha_t^2/\sigma_t^2$.

\subsection{Diffusion Process with Fixed Endpoints}

Diffusion models are limited because they can only transport complex data distributions to a standard Gaussian distribution and cannot be naturally adapted to translating between two arbitrary distributions, \eg in the case of image-to-image translation. Luckily, classical results have shown that one can condition a diffusion process on a fixed known endpoint via the famous Doob's $h$-transform:

% To solve this more general problem, the majority of previous works~\citep{ho2022classifier, rombach2022high, zhang2023adding, ho2022cascaded} resort to special architecture design or relies on inherent emergent properties of diffusion models~\citep{su2022dual}. 

% However, we seek a more principled solution by directly modifying the diffusion process to translate between arbitrary distributions, and we can orthogonally leverage the architectural design choices to aid our translation. Luckily, classical results have shown that when can condition a diffusion process on a fixed known endpoint via the $h$-transform, which marks our first step towards bridging two arbitrary distributions.

\mypara{Stochastic bridges via $h$-transform.}
Specifically, a diffusion process defined in \equref{eq:diffusion} can be driven to arrive at a particular point of interest $y \in \R^d$ almost surely via Doob's $h$-transform \citep{doob1984classical,rogers2000diffusions},
\begin{align}\label{eq:doob}
    d\rvx_t = \rvf(\rvx_t, t) dt + g(t)^2 \rvh(\rvx_t, t, y, T) + g(t) d\rvw_t, \quad \rvx_0\sim \pdata(\rvx),\quad \rvx_T = y
\end{align}
where $\rvh(x, t, y, T) = \nabla_{\rvx_t}\log p(\rvx_T\mid \rvx_t)\big \rvert_{\rvx_t = x, \rvx_T=y}$ is the gradient of the log transition kernel of from $t$ to $T$ generated by the original SDE, evaluated at points $\rvx_t=x$ and $\rvx_T=y$, and each $\rvx_t$ now explicitly depends on $y$ at time $T$. Furthermore, $p(\rvx_T=y\mid \rvx_t)$ satisfies the Kolmogorov backward equation (specified in \appref{sec:proof}). With specific drift and diffusion choices, \eg $\rvf(\rvx_t, t)=\vzero$, $\rvh$ is tractable due to the tractable (Gaussian) transition kernel of the underlying diffusion process.

% The modified diffusion process is now conditioned on $\rvx_T= y$. 
When the initial point $\rvx_0$ is fixed, 
the process is often called a \textit{diffusion bridge} \citep{sarkka2019applied,heng2021simulating,delyon2006simulation,schauer2017guided,peluchettinon,liu2022let}, and its ability to connect any given $\rvx_0$ to a given value of $\rvx_T$ is promising for image-to-image translation. Furthermore, the transition kernel may be tractable, which serves as further motivation.

% \mypara{Time-reversed diffusion bridges.}
% Prior works~\citep{heng2021simulating} have established that the time reversal of the diffusion bridge in \equref{eq:doob} with $\rvx_0=x, \rvx_T=y\in\R^d$ can be written as
% \begin{align}\label{eq:reverse-bridge}
%     d\rvx_t = \Big[\rvf(\rvx_t, t) - g^2(t) \Big(\rvs(\rvx_t, t, y, T) - \rvh(\rvx_t, t, y, T)\Big)\Big]dt + g(t) d\hat{\rvw}_t, \quad \rvx_0 = x_0,\; \rvx_T=y
% \end{align}
% where $\rvs(x, t, y, T) = \nabla_{\rvx_t}\log p(\rvx_t\mid \rvx_T)\big\rvert_{\rvx_t = x, \rvx_T=y}$ is the score of $\rvx_t$ conditioned on $\rvx_T$, evaluated at point $x$ and $y$ at time $t$ and $T$ respectively. Here $\hat{\rvw}_t$ denotes a Wiener process and $dt$ is a negative infinitesimal. This process is usually simulated via a simple method such as Euler-Maruyama given an approximation of $\nabla_{\rvx_t}\log p(\rvx_t\mid \rvx_T)$ is learned.
% \se{what's the relevance of time reversal here? why is it important?} 

%Prior works~\citet{liu2022let,peluchettinon} have explored related concepts and have shown that the diffusion generation process can be cast in the light of forward-time diffusion bridges. Our work differs in that we adopt a reverse-time perspective of the diffusion bridges and show that it gives rise to a strictly wider class of generative models that unify unconditional and conditional diffusion models. We also show that flow-based image-to-image translation models such as those introduced by \citet{liu2022flow} can be seen as a subclass of our method.

\section{Denoising Diffusion Bridge Models}
\label{sec:methods}

\begin{figure}
    \centering
    \includegraphics[width=0.85\textwidth]{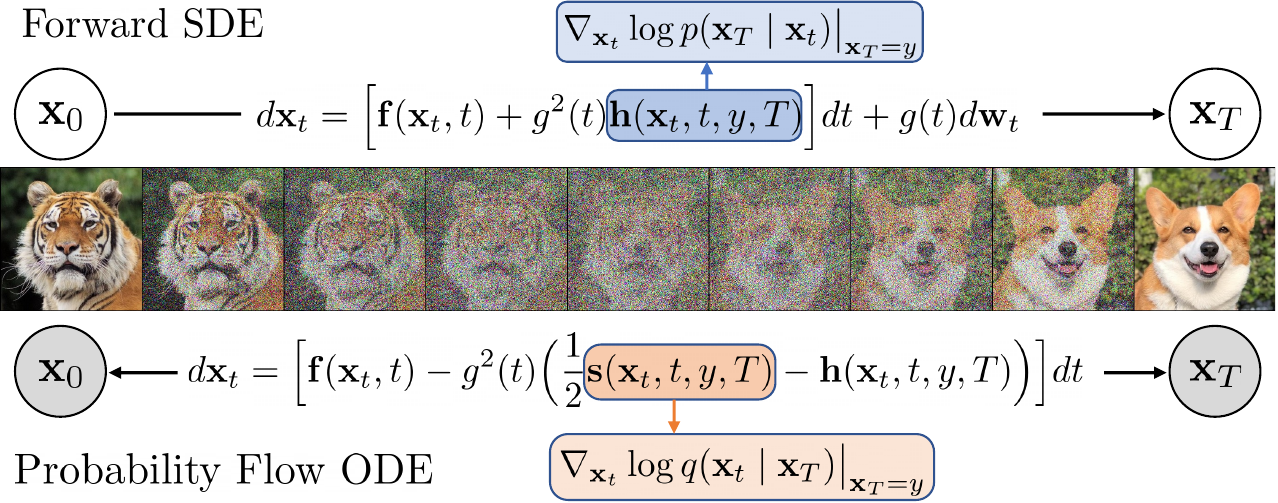}
    \caption{\textbf{A schematic for Denoising Diffusion Bridge Models.} DDBM uses a diffusion process guided by a drift adjustment (in blue) towards an endpoint $\rvx_T = y$. They lears to reverse such a bridge process by matching the denoising bridge score (in orange), which allows one to reverse from $\rvx_T$ to $\rvx_0$ for any $\rvx_T=\rvy\sim\pdata(\rvy)$. The forward SDE process shown on the top is unidirectional while the probability flow ODE shown at the bottom is deterministic and bidirectional. White nodes are stochastic while grey nodes are deterministic.}
    \label{fig:bridge}
    \vspace{-3mm}
\end{figure}

% While prior works~\citet{liu2022let,peluchettinon} have shown that unconditional generation processes can be re-cast in the light of forward-time diffusion bridges, our work emphasizes the reverse-time perspective of diffusion bridges as a unifying framework for both unconditional generation and image-to-image translation. This is important because one can transfer many recent advances in unconditional diffusion models to the more general problem of translating between arbitrary distributions, as we do in this work.

Assuming that the endpoints of a diffusion bridge both exist in $\R^d$ and come from an arbitrary and unknown joint distribution, \ie $(\rvx_0, \rvx_T) =(\rvx, \rvy)\sim \pdata(\rvx,\rvy)$, we wish to devise a process that learns to approximately sample from $\pdata(\rvx\mid \rvy)$ by reversing the diffusion bridge with boundary distribution $\pdata(\rvx,\rvy)$, given a training set of \emph{paired} samples drawn from $\pdata(\rvx,\rvy)$.

\subsection{Time-Reversed SDE and Probability Flow ODE}\label{sec:time-rev}
% \se{logic gap here, why are we suddenly talking about efficiency?}
Inspired by diffusion bridges, we construct the stochastic process $\{\rvx_t\}_{t=0}^{T}$ with marginal distribution $q(\rvx_t)$ such that $q(\rvx_0, \rvx_T)$ approximates $\pdata(\rvx_0,\rvx_T)$. Reversing the process amounts to sampling from $q(\rvx_t\mid \rvx_T)$. Note that distribution $q(\cdot)$ is different from $p(\cdot)$, \ie the diffusion marginal distribution, in that the endpoint distributions are now $\pdata(\rvx_0, \rvx_T) = \pdata(\rvx, \rvy)$ instead of the distribution of a diffusion $p(\rvx_0, \rvx_T) = p(\rvx_T\mid \rvx_0)\pdata(\rvx_0)$, which defines a Gaussian $\rvx_T$ given $\rvx_0$. 
%A schematic of the bridge process is shown in \figref{fig:bridge}. 
%As the bridge process is based on diffusion processes, we can further leverage important implementation choices and advancements in diffusion models such as different noise schedules and efficient sampling~\citep{song2020denoising, karras2022elucidating, lu2022dpm,lu2022dpm2,zhang2022fast} via exponential integration of the diffusion probability flow ODE. 
We can construct the time-reversed SDE/probability flow ODE of $q(\rvx_t\mid \rvx_T)$ via the following theorem.
\begin{restatable}[]{theorem}{revbridge}\label{thm:revbridge}
The evolution of conditional probability $q(\rvx_t\mid \rvx_T)$  has a time-reversed SDE of the form
\begin{align}\label{eq:bridgesde}
    d\rvx_t = \Big[\rvf(\rvx_t, t) - g^2(t) \Big(\rvs(\rvx_t, t, y, T) - \rvh(\rvx_t, t, y, T)\Big)\Big]dt + g(t) d\hat{\rvw}_t, \quad \rvx_T = y
\end{align}
with an associated probability flow ODE
\begin{align}\label{eq:bridgeode}
    d\rvx_t = \Big[\rvf(\rvx_t, t) - g^2(t) \Big( \frac{1}{2}\rvs(\rvx_t, t, y, T) - \rvh(\rvx_t, t, y, T)\Big)\Big]dt,\quad \rvx_T = y
\end{align}
on $t\le T-\epsilon$ for any $\epsilon>0$, where $\hat{\rvw}_t$ denotes a Wiener process, $\rvs(x, t, y, T) = \gradnd{\log q(\rvx_t\mid \rvx_T)}{\rvx_t}\big \vert_{\rvx_t = x, \rvx_T = y}$ and $\rvh$ is as defined in \equref{eq:doob}.
\end{restatable}
A schematic of the bridge process is shown in \figref{fig:bridge}. Note that this process is defined up to $T-\epsilon$. To recover the initial distribution in the SDE case, we make an approximation that $\rvx_{T-\epsilon}\approx y$ for some small $\epsilon$ simulate SDE backward in time. For the ODE case, since we need to sample from $p(\rvx_{T-\epsilon})$ which cannot be Dirac delta, we cannot approximate $\rvx_{T-\epsilon}$ with a single $y$. Instead, we can first approximate $\rvx_{T-\epsilon'}\approx y$ where $\epsilon>\epsilon' > 0$, and then take an Euler-Maruyama step to $\rvx_{T-\epsilon}$, and \equref{eq:bridgeode} can be used afterward.
%We can observe that the additional drift adjustment $\rvh(x, t, y, T) = \nabla_{\rvx_t}\log p(\rvx_T\mid \rvx_t)\big \rvert_{\rvx_t = x, \rvx_T=y}$ ensures that the reversal process "pulls" any samples towards $y$ when simulating forward in time, or pushes away from $y$ backwards in time, similar to classifier-guidance~\citep{dhariwal2021diffusion, ho2022classifier} in spirit. 
A toy visualization of VE bridge and VP bridges are shown in \figref{fig:bridge-example}. The top and bottom shows the respective SDE and ODE paths for VE and VP bridges. %One can notice that VP bridges result in blurry images in between due to the signal-destroying properties of the underlying VP diffusion. 

\begin{figure}[t]
    \centering
    \begin{subfigure}[t]{0.45\linewidth}
        \centering
        \includegraphics[width=\linewidth]{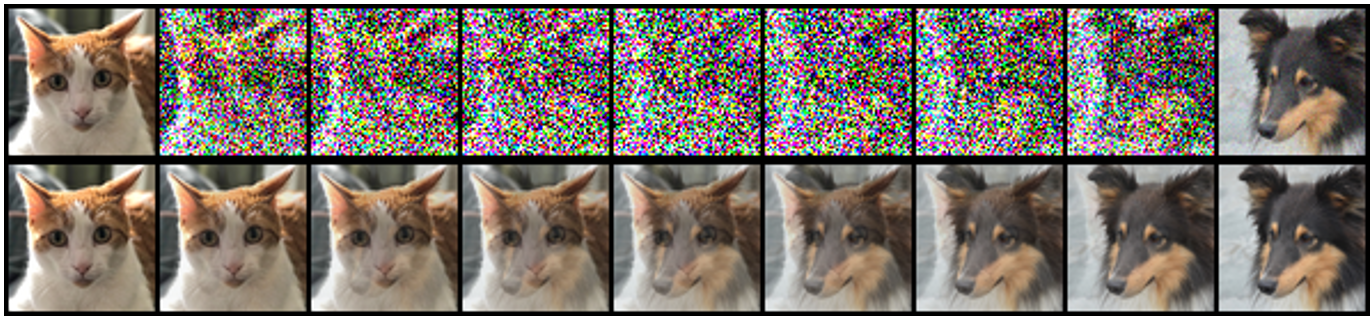}
        \label{fig:ve-interp}
    \end{subfigure}
    \hspace{0.02\linewidth}
    \begin{subfigure}[t]{0.45\linewidth}
        \centering
        \includegraphics[width=\linewidth]{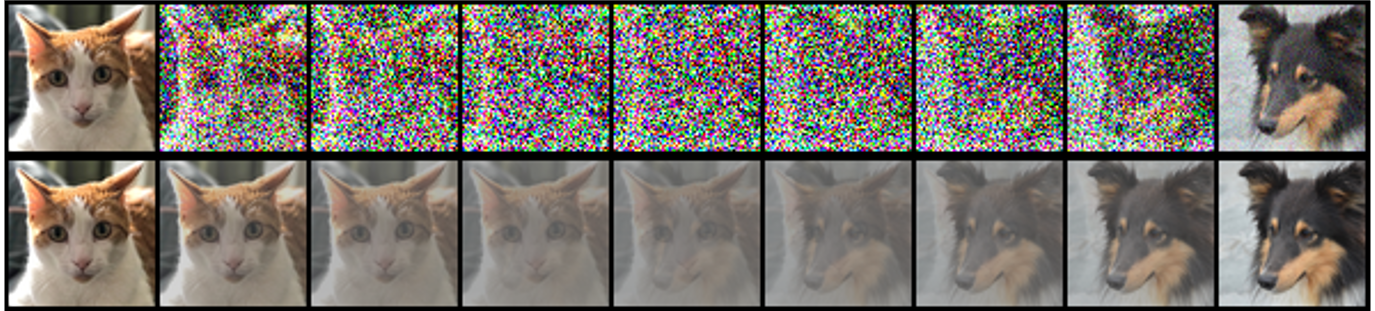}
        \label{fig:vp-interp}
    \end{subfigure}
    \vspace{-5mm}
    \caption{VE bridge (left) and VP bridge (right) with their SDE (top) and ODE (bottom) visualization. }
    \label{fig:bridge-example}
    \vspace{-3mm}
\end{figure}

% Furthermore, due to the resemblance with classifier-guidance, we can introduce an additional parameter $w$ to set the "strength" of drift adjustment as below.
% \begin{align}\label{eq:guidance-ode}
%     d\rvx_t = \Big[\rvf(\rvx_t, t) - g^2(t) \Big( \frac{1}{2}\rvs(\rvx_t, t, y, T) - w\rvh(\rvx_t, t, y, T)\Big)\Big]dt,\quad \rvx_T = y
% \end{align}
% which allows for a strictly wider class of marginal density of $\rvx_t$ generated by the resulting probability flow ODE. 

\subsection{Marginal Distributions and Denoising Bridge Score Matching}

The sampling process in \thmref{thm:revbridge} requires approximation of the score $\rvs(x, t, y, T) = \gradnd{\log q(\rvx_t\mid \rvx_T)}{\rvx_t}\big \vert_{\rvx_t = x, \rvx_T = y}$ where $q(\rvx_t\mid\rvx_T)=\int_{\rvx_0}q(\rvx_t\mid \rvx_0, \rvx_T)\pdata(\rvx_0\mid\rvx_T)d\rvx_0$. However, as the true score is not known in closed-form, we take inspiration from denoising score-matching~\citep{song2020score} and use a neural network to approximate the true score by matching against a tractable quantity. This usually results in closed-form marginal sampling of $\rvx_t$ given data (\eg $\rvx_0$ in the case of diffusion models and $(\rvx_0, \rvx_T)$ in our case), and given $\rvx_t$, the model is trained to match against the closed-form denoising score objective. We are motivated to follow a similar approach because (1) tractable marginal sampling of $\rvx_t$ and (2) closed-form objectives enable a simple and scalable algorithm. We specify how to design the marginal sampling distribution and the tractable score objective below to approximate the ground-truth conditional score $\gradnd{\log q(\rvx_t\mid \rvx_T)}{\rvx_t}$.
%We seek to approach our objective with the following motivation: (1) a scalable algorithm should allow one to efficiently sample intermediate $\rvx_t$ and (2) have a tractable and simple objective. Fortunately, for the former condition, we can design our sampling distribution $q(\cdot)$ such that $q(\rvx_t\mid \rvx_0, \rvx_T) \defeq p(\rvx_t\mid \rvx_0, \rvx_T )$, where $p(\cdot)$ is the distribution of a diffusion process when pinned at both endpoints as in the case of \equref{eq:doob}.
%We wish to learn the score of $q(\rvx_t\mid\rvx_T)=\int_{\rvx_0}q(\rvx_t\mid \rvx_0, \rvx_T)\pdata(\rvx_0\mid\rvx_T)d\rvx_0$ as defined in the previous section, and we can further choose the drift $\rvf(\rvx_t, t)$ and diffusion coefficient $g(t)$ so that the transition kernel of \equref{eq:doob} is tractable, \ie Gaussian~\citep{song2020score}. 
% Note that under these assumptions
% % %this tractability assumption automatically allows 
%  the drift adjustment $\rvh(x, t, y, T) = \nabla_{\rvx_t}\log p(\rvx_T\mid \rvx_t)\big \rvert_{\rvx_t = x, \rvx_T=y}$ in \equref{eq:doob} is also analytically tractable.

\begin{table}[t]
    \centering
\resizebox{0.98\linewidth}{!}{
    \begin{tabular}{cccccc}
    \toprule
     & $\rvf(\rvx_t, t)$ & $g^2(t)$ & $p(\rvx_t\mid \rvx_0)$ & $\text{SNR}_t$ & $\nabla_{\rvx_t}\log p(\rvx_T\mid \rvx_t)$ \\
     \midrule
        VP & $\frac{d\log\alpha_t }{dt}\rvx_t$ & $\frac{d}{dt}\sigma_t^2 - 2\frac{d\log\alpha_t}{dt}  \sigma_t^2$ & $\gN(\alpha_t \rvx_0, \sigma_t^2\mI)$ & $\alpha_t^2/\sigma_t^2$ & $\frac{(\alpha_t /\alpha_T) \rvx_T - \rvx_t}{\sigma_t^2 (\text{SNR}_t/\text{SNR}_T - 1)}$\\
     \midrule
         VE & $\vzero$ & $\frac{d}{dt}\sigma_t^2$ & $\gN(\rvx_0, \sigma_t^2\mI)$& $1/\sigma_t^2$ & $\frac{\rvx_T - \rvx_t}{ \sigma_T^2 - \sigma_t^2}$ \\
     \bottomrule
    \end{tabular}
    }
    \caption{VP and VE instantiations of diffusion bridges.}
    \label{tab:bridge_ins}
    \vspace{-5mm}
\end{table}

\mypara{Sampling distribution.} Fortunately, for the former condition, we can design our sampling distribution $q(\cdot)$ such that $q(\rvx_t\mid \rvx_0, \rvx_T) \defeq p(\rvx_t\mid \rvx_0, \rvx_T )$, where $p(\cdot)$ is the diffusion distribution pinned at both endpoints as in \equref{eq:doob}. For diffusion processes with Gaussian transition kernels, \eg VE, VP~\citep{song2020score}, our sampling distribution is a Gaussian distribution of the form
\begin{equation}\label{eq:marginal}
\begin{aligned}
    q(\rvx_t\mid \rvx_0, \rvx_T) &= \gN(\hat{\mu}_t, \hat{\sigma}_t^2\bm{I}),\quad \text{where}\\ 
    \hat{\mu}_t&= \frac{\text{SNR}_T}{\text{SNR}_t} \frac{\alpha_t}{\alpha_T}\rvx_T + \alpha_t \rvx_0(1-\frac{\text{SNR}_T}{\text{SNR}_t})\\
    \hat{\sigma}_t^2&=\sigma_t^2(1-\frac{\text{SNR}_T}{\text{SNR}_t})
\end{aligned}
\end{equation}
where $\alpha_t$ and $\sigma_t$ are pre-defined signal and noise schedules and $\normalfont\text{SNR}_t = \alpha_t^2/\sigma_t^2$ is the signal-to-noise ratio at time $t$. For VE schedule, we assume $\alpha_t=1$ and derivation details are provided in \appref{sec:marginal-proof}. Notably, the mean of this distribution is a linear interpolation between the (scaled) endpoints, and the distribution approaches a Dirac distribution when nearing either end. For concreteness, we present the bridge processes generated by both VP and VE diffusion in \tabref{tab:bridge_ins} and recommend choosing $\rvf$ and $g$ specified therein.

\mypara{Training objective.} For the latter condition, diffusion bridges benefit from a similar setup as in diffusion models, since a pre-defined signal/noise schedule gives rise to a closed-form conditional score $\gradnd{\log q(\rvx_t\mid\rvx_0, \rvx_T)}{\rvx_t}$. We show in the following theorem that with $\rvx_t\sim q(\rvx_t \mid \rvx_0, \rvx_T)$, a neural network $\rvs_\theta(\rvx_t, \rvx_T, t)$ that matches against this closed-form score approximates the true score. 
\begin{restatable}[Denoising Bridge Score Matching]{theorem}{naiveobj}\label{thm:naiveobj}
Let $(\rvx_0, \rvx_T)\sim \pdata(\rvx, \rvy)$, $\rvx_t\sim q(\rvx_t\mid \rvx_0, \rvx_T)$, $t\sim p(t)$ for any non-zero time sampling distribution $p(t)$ in $[0,T]$, and $w(t)$ be a non-zero loss weighting term of any choice. Minimum of the following objective:
\begin{align}\label{eq:naive-obj}
\mathcal{L}(\theta) &= \E_{\rvx_t, \rvx_0, \rvx_T, t}\Big[ w(t) \norm{\rvs_\theta(\rvx_t, \rvx_T, t) - \gradnd{\log q(\rvx_t\mid\rvx_0, \rvx_T)}{\rvx_t}}^2 \Big]
\end{align}
satisfies $\rvs_\theta(\rvx_t, \rvx_T, t) = \gradnd{\log q(\rvx_t\mid \rvx_T)}{\rvx_t}$.
\end{restatable}

In short, we establish a tractable diffusion bridge over two endpoints and, by matching the conditional score of the Gaussian bridge, we can learn the score of the new distribution $q(\rvx_t\mid \rvx_T)$ that satisfies the boundary distribution $\pdata(\rvx, \rvy)$. %We illustrate how we can sample using $q(\rvx_t\mid \rvx_T)$ in the following section.

\section{Generalized Parameterization for Distribution Translation}

Building the bridge process upon diffusion process allows us to further adapt many recent advancements in the score network parameterization $\rvs_\theta(\rvx_t, \rvx_T, t)$~\citep{ho2020denoising, song2020score,salimans2022progressive, ho2022imagen, karras2022elucidating}, different noise schedules, and efficient ODE sampling~\citep{song2020denoising, karras2022elucidating, lu2022dpm,lu2022dpm2,zhang2022fast} to our more general framework. Among these works, EDM~\citep{karras2022elucidating} proposes to parameterize the model output to be $D_\theta(\rvx_t, t) = c_{\text{skip}}(t)\rvx_t + c_{\text{out}}(t)F_\theta(c_{\text{in}}(t) \rvx_t, c_{\text{noise}}(t) )$ where $F_\theta$ is a neural network with parameter $\theta$ that predicts $\rvx_0$. 
%and shows that many existing parameterization schemes fall as special cases with different choices of the scaling functions $c_{\text{skip}}(t), c_{\text{out}}(t), c_{\text{in}}(t), c_{\text{noise}}(t)$. 
In a similar spirit, %we are also motivated to adopt this parameterization for image translation because (1) the bridge score-matching objective can be re-written to predict $\rvx_0$, (2) the variance of prediction target $\rvx_0$ does not change with time and (3) optimal scaling functions minimize sensitivity to changes in variance of both model inputs and outputs. In this section, 
we adopt this pred-$\rvx$ parameterization and additionally derive a set of 
%generalized optimal
scaling functions for distribution translation, which we show is a strict superset.

\mypara{Score reparameterization.} 
% For simplicity, we also adopt the VE schedule explored by EDM with variance $\sigma_t^2 = t^2$. It is easy to see that score of \equref{eq:marginal} is $-(\rvx_t - \hat{\mu}_t) / \hat{\sigma}_t^2$, and matching this score given $\rvx_t$ and $\rvx_T$ is equivalent to matching $\rvx_0$ with loss
% \begin{align}
%     \gL_{\text{pred-$\rvx$}}(\theta) = \E_{\rvx_t, \rvx_0, \rvx_T, t}\Big[ \Tilde{w}(t) \norm{D_\theta(\rvx_t, \rvx_T, t) - \rvx_0}^2 \Big]
% \end{align} 
% where $\rvx_t, \rvx_0, \rvx_T, t$ are distributed as in \equref{eq:naive-obj} and $\Tilde{w}(t)$ is a weighting function. 
Following the sampling distribution proposed in~\eqref{eq:marginal}, a pred-$\rvx$ model can predict bridge score by
\begin{align}
    \nabla_{\rvx_t}\log q(\rvx_t\mid \rvx_T) \approx -\frac{\rvx_t - \Big(\frac{\text{SNR}_T}{\text{SNR}_t} \frac{\alpha_t}{\alpha_T}\rvx_T + \alpha_t D_\theta(\rvx_t,\rvx_T, t)(1-\frac{\text{SNR}_T}{\text{SNR}_t})\Big)}{\sigma_t^2(1-\frac{\text{SNR}_T}{\text{SNR}_t})}
\end{align}

\mypara{Scaling functions and loss weighting.} Following \citet{karras2022elucidating}, and let $a_t = \alpha_t/\alpha_T* \text{SNR}_T/\text{SNR}_t$, $b_t = \alpha_t (1 -  \text{SNR}_T/\text{SNR}_t)$, $c_t = \sigma_t^2 (1- \text{SNR}_T/\text{SNR}_t)$, the scaling functions and weighting function $w(t)$ can be derived to be
\begin{align}
    &c_{\text{in}}(t) = \frac{1}{\sqrt{a_t^2 \sigma_T^2 + b_t^2 \sigma_0^2 + 2a_t b_t
\sigma_{0T} + c_t )}}, \quad
c_{\text{out}}(t) = \sqrt{a_t^2(\sigma_T^2\sigma_0^2 - \sigma_{0T}^2) + \sigma_0^2 c_t} * c_{\text{in}}(t)\\
 &c_{\text{skip}}(t) = \Big(b_t \sigma_0^2 + a_t \sigma_{0T} \Big) * c_{\text{in}}^2(t), \quad  w(t) = \frac{1}{c_{\text{out}}(t)^2}, \quad c_{\text{noise}}(t) = \frac{1}{4}\log{(t)}
\end{align}
where $\sigma_0^2$, $\sigma_T^2$, and $\sigma_{0T}$ denote the variance of $\rvx_0$, variance of $\rvx_T$, and the covariance of the two, respectively. The only additional hyperparameters compared to EDM are $\sigma_T$ and $\sigma_{0T}$, which characterize the distribution of $\rvx_T$ and its correlation with $\rvx_0$. One can notice that in the case of EDM, $\sigma_t = t$, $\sigma_T^2 = \sigma_0^2 + T^2$ because $\rvx_T = \rvx_0 + T \rvepsilon$ for some Gaussian noise $\rvepsilon$, $\sigma_{0T} = \sigma_0^2$, and $\text{SNR}_T/\text{SNR}_t = t^2 / T^2$. One can show that the scaling functions then reduce to those in EDM. We leave details in \appref{sec:edm-proof}.

% below to appendix?
% In general, one can observe that, in the case of image-to-image translation where we can assume $\sigma_0 = \sigma_T$, the choices of $\sigma_{0T}$ can influence how much information about $\rvx_T$ is retained in the model output, \ie when $\sigma_{0T} = 0$, $c_{\text{skip}}(T) = 0, c_{\text{out}}(T) = 1$ and $c_{\text{skip}}(0) = 1, c_{\text{out}}(0) = 0$ while when $\sigma_{0T} = 0$, and when $\sigma_{0T} = \sigma_0^2$, $c_{\text{skip}}(T) = c_{\text{skip}}(0) = 1$ and $c_{\text{out}}(T) = c_{\text{out}}(0) = 0$. If $\rvx_0$ and $\rvx_T$ are reasonably correlated, one can choose $\sigma_{0T} = \sigma_{0}^2/2$ such that  $c_{\text{skip}}(T) = c_{\text{out}}(T) = 0.5$ which allows the output to retain some visual information about $\rvx_T$ at time $T$.

\mypara{Generalized time-reversal.} Due to the probability flow ODE's resemblance with classifier-guidance~\citep{dhariwal2021diffusion, ho2022classifier}, we can introduce an additional parameter $w$ to set the "strength" of drift adjustment as below.
\begin{align}\label{eq:guidance-ode}
    d\rvx_t = \Big[\rvf(\rvx_t, t) - g^2(t) \Big( \frac{1}{2}\rvs(\rvx_t, t, y, T) - w\rvh(\rvx_t, t, y, T)\Big)\Big]dt,\quad \rvx_T = y
\end{align}
which allows for a strictly wider class of marginal density of $\rvx_t$ generated by the resulting probability flow ODE. We examine the effect of this parameter in our ablation studies.

\section{Stochastic Sampling for Denoising Diffusion Bridges}

\begin{algorithm}[t]
\caption{Denoising Diffusion Bridge Hybrid Sampler}\label{alg:sampler}
\begin{algorithmic}
\State {\bf Input:} model $D_\theta(\rvx_t, t)$, time steps $\{t_i\}_{i=0}^{N}$, max time $T$, guidance strength $w$, step ratio $s$, distribution $\pdata(\rvy)$
\State {\bf Output:} $\rvx_0$
\State {\bf Sample} $\rvx_{N}\sim \pdata(\rvy)$
\For{$i = N,\dots, 1$}

\State {\bf Sample} $\rvepsilon_i \sim \gN(\vzero, \vI)$
\State $\hat{t}_i \gets$ $t_i + s(t_{i-1} - t_i)$
\State $\bm{d}_i\gets$ $-\rvf(\rvx_i, t_i) + g^2(t_i) \Big(\rvs(\rvx_i, t_i, \rvx_N, T) - \rvh(\rvx_i, t_i, \rvx_N, T)\Big)$
\State $\hat{\rvx}_i \gets \rvx_i + \bm{d}_i(\hat{t}_i - t_i) + g(t_i) \sqrt{\hat{t}_i - t_i} \rvepsilon_i$

\State $\hat{\bm{d}}_i \gets -\rvf(\hat{\rvx}_i, \hat{t}_i) + g^2(\hat{t}_i) \Big(\frac{1}{2}\rvs(\hat{\rvx}_i, \hat{t}_i, \rvx_N, T) - w \rvh(\hat{\rvx}_i, \hat{t}_i, \rvx_N, T)\Big)$
\State $\rvx_{i-1} \gets \hat{\rvx}_i + \hat{\bm{d}}_i (t_{i-1} - \hat{t}_i)$
\If{$i \neq 1$}
    \State $\bm{d}'_i \gets -\rvf(\rvx_{i-1}, t_{i-1}) + g^2(t_{i-1}) \Big(\frac{1}{2}\rvs(\rvx_{i-1}, t_{i-1}, \rvx_N, T) - w \rvh(\rvx_{i-1}, t_{i-1}, \rvx_N, T)\Big)$
    \State $\rvx_{i-1} \gets \hat{\rvx}_i + (\frac{1}{2}\bm{d}'_i + \frac{1}{2}\hat{\bm{d}}_i)(t_{i-1} - \hat{t}_i) $
\EndIf
\EndFor
\end{algorithmic}
\end{algorithm}

Although the probability flow ODE allows for one to use fast integration techniques to accelerate the sampling process~\citep{zhang2022fast,song2020denoising,karras2022elucidating}, purely following an ODE path is problematic because diffusion bridges have fixed starting points given as data $\rvx_T = \rvy \sim \pdata(\rvy)$, and following the probability flow ODE backward in time generates a deterministic "expected" path. This can result in "averaged" or blurry outputs given initial conditions. Thus, we are motivated to introduce noise into our sampling process to improve the sampling quality and diversity. 

% \mypara{Higher-order deterministic sampler.} Our sampler is built upon prior deterministic samplers. In particular, the most related to ours is the Heun sampler introduced by \citep{karras2022elucidating}, which first discretizes the sampling steps into $t_0 < t_1 \dots < t_N$ where
% \begin{align}
%     t_{i>0} = \Big(T^{\frac{1}{\rho}} + \frac{N - i}{N-1}(t_{\text{min}}^{\frac{1}{\rho}} - T^{\frac{1}{\rho}})\Big)^\rho \quad \text{and} \quad t_0 = 0
% \end{align}
% and $\rho = 7$ is a default choice. It then integrates over the probability flow ODE path with second-order Heun steps for each such discretization step. We refer readers to \citet{karras2022elucidating} for algorithm details.

\mypara{Higher-order hybrid sampler.} Our sampler is built upon prior higher-order ODE sampler in \citep{karras2022elucidating}, which discretizes the sampling steps into $t_N > t_{N-1} > \dots > t_0$ with decreasing intervals (see \appref{sec:discretization} for details).  Inspired by the predictor-corrector sampler introduced by \citet{song2020score}, we additionally introduce a scheduled Euler-Maruyama step which follows the backward SDE in between higher-order ODE steps. This ensures that the marginal distribution at each step approximately stays the same. We introduce additional scaling hyperparameter $s$, which define a step ratio in between $t_{i-1}$ and $t_{i}$ such that the interval $[t_{i} - s(t_{i} - t_{i-1}), t_{i}]$ is used for Euler-Maruyama steps and $[t_{i-1}, t_{i} - s(t_{i} - t_{i-1})]$ is used for Heun steps, as described in \algoref{alg:sampler}.

\section{Related Works and Special Cases} 

\mypara{Diffusion models.} The advancements in diffusion models~\citep{sohl2015deep,ho2020denoising,song2020score} have improved state-of-the-art in image generation and outperformed GANs~\citep{Goodfellow2014GenerativeAN}. The success of diffusion models goes hand in hand with important design choices such as network design~\citep{song2020score,karras2022elucidating,nichol2021improved,hoogeboom2023simple,peebles2023scalable}, improved noise-schedules~\citep{nichol2021improved,karras2022elucidating,peebles2023scalable}, faster and more accurate samplers~\citep{song2020denoising,lu2022dpm,lu2022dpm2,zhang2022fast}, and guidance methods~\citep{dhariwal2021diffusion,ho2022classifier}. Given the large body of literature on diffusion models for unconditional generation, which largely is based on these various design choices, we seek to design our bridge formulation to allow for a seamless integration with this literature. As such, we adopt a time-reversal perspective to directly extend these methods.

\mypara{Diffusion bridges, Sch\"odinger bridges, and Doob's h-transform.}
Diffusion bridges~\citep{sarkka2019applied} are a common tool in probability theory and have been actively studied in recent years in the context of generative modeling~\citep{liu2022let, somnath2023aligned,de2021diffusion,peluchettinon,peluchetti2023diffusion}. \citet{heng2021simulating} explores diffusion bridges conditioned on fixed starting/ending points and learns to simulate the time-reversal of the bridge given an approximation of the score $\nabla_{\rvx_t}\log p(\rvx_t)$. More recently, instead of considering bridges with fixed endpoints, \citet{liu2022let} uses Doob's h-transform to bridge between arbitrary distributions. A forward bridge is learned via score-matching by simulating entire paths during training. In contrast, other works~\citep{somnath2023aligned,peluchettinon}, while also adopting Doob's h-transform, propose simulation-free algorithms for forward-time generation. \citet{delbracio2023inversion} similarly constructs a Brownian Bridge for direct iteration and is successfully applied to image-restoration tasks. Another approach~\cite{de2021diffusion} proposes Iterative Proportional Fitting (IPF) to tractably solve Sch\"odinger Bridge (SB) problems in translating between different distributions. \citet{liu2023i2sb} is built on a tractable class of SB which results in a simulation-free algorithm and has demonstrated strong performance in image translation tasks. More recently, extending SB with IPF, Bridge-Matching~\citep{shi2023diffusion} proposes to use Iterative Markovian Fitting to solve the SB problem. A similar algorithm is also developed by \citet{peluchetti2023diffusion} for distribution translation. A more related work to ours is \citet{li2023bbdm}, which proposes to directly reverse a Brownian Bridge for distribution translation in discrete time.  Our method instead shows how to construct a bridge model from any existing VP and VE diffusion processes in continuous time, and Brownian Bridge (as considered in most previous works) is but a special case of VE bridges. We additionally show that, when implemented correctly, VP bridges can achieve very strong empirical performance. Although similar perspective can be derived using forward-time diffusion as in ~\citet{peluchettinon} which also proposes VE/VP bridge schedules, our framework enjoys additional empirical (reusing diffusion designs) and theoretical (connection with OT-Flow-Matching\citep{lipman2023flow, tong2023improving} and Rectified Flow~\citep{liu2022flow}) benefits.

\mypara{Flow and Optimal Transport} Works based on Flow-Matching~\citep{lipman2023flow, tong2023improving, pooladian2023multisample, tong2023simulation} learn an ODE-based transport map to bridge two distributions. \citet{lipman2023flow} has demonstrated that by matching the velocity field of predefined transport maps, one can create powerful generative models competitive with the diffusion counterparts. Improving this approach, \citet{tong2023improving,pooladian2023multisample} exploit potential couplings between distributions using minibatch simulation-free OT. Rectified Flow~\citep{liu2022flow} directly constructs the OT bridge and uses neural networks to fit the intermediate velocity field. Another line of work uses stochastic interpolants \citep{albergo2023building} to build flow models and directly avoid the use of Doob's h-functions and provide an easy way to construct interpolation maps between distributions. \citet{albergo2023stochastic} presents a general theory with stochastic interpolants unifying flow and diffusion, and shows that a bridge can be constructed from both an ODE and SDE perspctive. Separate from these methods, our model uses a different denoising bridge score-matching loss than this class of models. Constructing from this perspective allows us to extend many existing successful designs of diffusion models (which are not directly applicable to these works) to the bridge framework and push state-of-the-art further for image translation while retaining strong performance for unconditional generation.

\subsection{Special Cases of Denoising Diffusion Bridge Models}
%Relaxing diffusion processes to an arbitrary prior significantly increases the capacity of our model. 
% One natural question is what special cases our framework contains because, by leveraging insights from special cases, we can improve our modeling. Towards this end, we show that our reverse-time perspective of denoising diffusion bridges can now naturally subsume different types of models, thus unifying many existing methods on both unconditional generation and image-to-image translation.

\mypara{Case 1: Unconditional diffusion process~\citep{song2020score}.} For unconditional diffusion processes (which map data to noise), we can first show that the marginal $p(\rvx_t)$ when $p(\rvx_0)=\pdata(\rvx)$ exactly matches that of a regular diffusion process when $\rvx_T\sim \pdata(\rvy\mid\rvx)=\gN(\alpha_T\rvx,\sigma_T^2\mI)$. By taking expectation over $\rvx_T$ in \equref{eq:marginal}, we have
\begin{align}
    p(\rvx_t\mid \rvx_0) = \gN(\alpha_t\rvx_0, \sigma_t\mI)
\end{align}
One can further show that during sampling, \twoequref{eq:bridgesde}{eq:bridgeode} reduce to the reverse SDE and ODE (respectively) of a diffusion process when $\rvx_T$ is sampled from a Gaussian. We leave derivation details to \appref{sec:special-case-proof}. 

\mypara{Case 2: OT-Flow-Matching~\citep{lipman2023flow,tong2023improving} and Rectified Flow~\citep{liu2022flow}.} These works learn to match deterministic dynamics defined through ODEs instead of SDEs. In this particular case, they work with ``straight line" paths defined by $\rvx_T - \rvx_0$.

To see that our framework generalizes this, first let us define a family of diffusion bridges with variance scaled by $c\in (0,1)$ such that  $p(\rvx_t\mid \rvx_0, \rvx_T) = \gN(\hat{\mu}_t, c^2\hat{\sigma}_t^2\bm{I})$ where $\hat{\mu}_t$ and $\hat{\sigma}_t$ are as defined in \equref{eq:marginal}. One can therefore show that with a VE diffusion where $\sigma_t^2 = c^2 t$, given some fixed $\rvx_0$ and $\rvx_t$, \ie $T=1$, and $\rvx_t$ sampled from \equref{eq:marginal},
\begin{align}
    \lim_{c\to 0} \Big[\rvf(\rvx_t, t) - c^2g^2(t) \Big(\frac{1}{2} \gradnd{\log p(\rvx_t\mid\rvx_0, \rvx_1)}{\rvx_t} - \gradnd{\log p(\rvx_1\mid\rvx_t)}{\rvx_t})\Big)\Big]
    = \rvx_1 - \rvx_0
\end{align}
where inside the bracket is the drift of probability flow ODE in \equref{eq:bridgeode} given $\rvx_0$ and $\rvx_1$, and the right hand side is exactly the straight line path term. In other words, these methods learn to match the drift in the bridge probability flow ODE (with a specific VE schedule) in the noiseless limit. The score model can then be matched against $\rvx_T - \rvx_0$, with some additional caveat to handle additional input $\rvx_T$, our framework exactly reduces to that of OT-Flow-Matching and Rectified Flow (details in \appref{sec:special-case-proof}).
\section{Experiments}

In this section we verify the generative capability of \model, and we want to answer the following questions: (1) How well does \model perform in image-to-image translation in pixel space? (2) Can \model perform well in unconditional generation when one side of the bridge reduces to Gaussian distribution? (3) How does the additional design choices introduced affect the final performance? Unless noted otherwise, we use the same VE diffusion schedule as in  EDM for our bridge model by default. We leave further experiment details to \appref{sec:exp-details}.

\begin{figure}[t]
    \centering
    \setlength\tabcolsep{0.6pt}
    \begin{tabular}{c|cccccc}
      \begin{subfigure}[t]{0.138\linewidth}
        \centering
        \includegraphics[width=\linewidth]{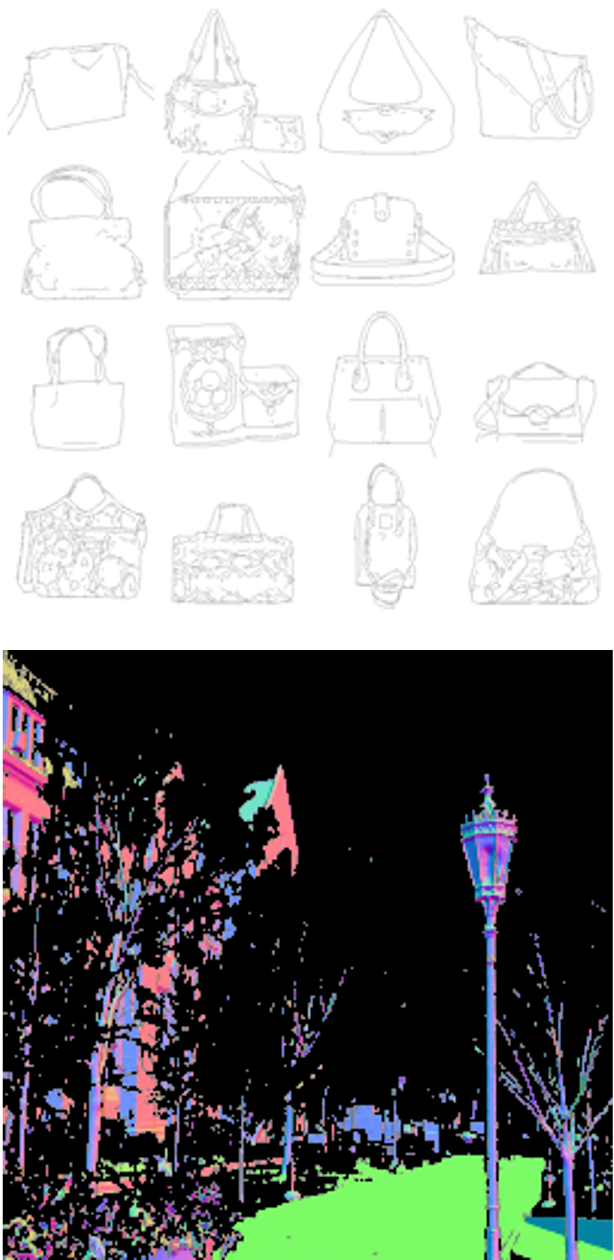}
    \end{subfigure}
    \hspace{0.0002\linewidth} &
    \hspace{0.0002\linewidth} 
    \begin{subfigure}[t]{0.138\linewidth}
        \centering
        \includegraphics[width=\linewidth]{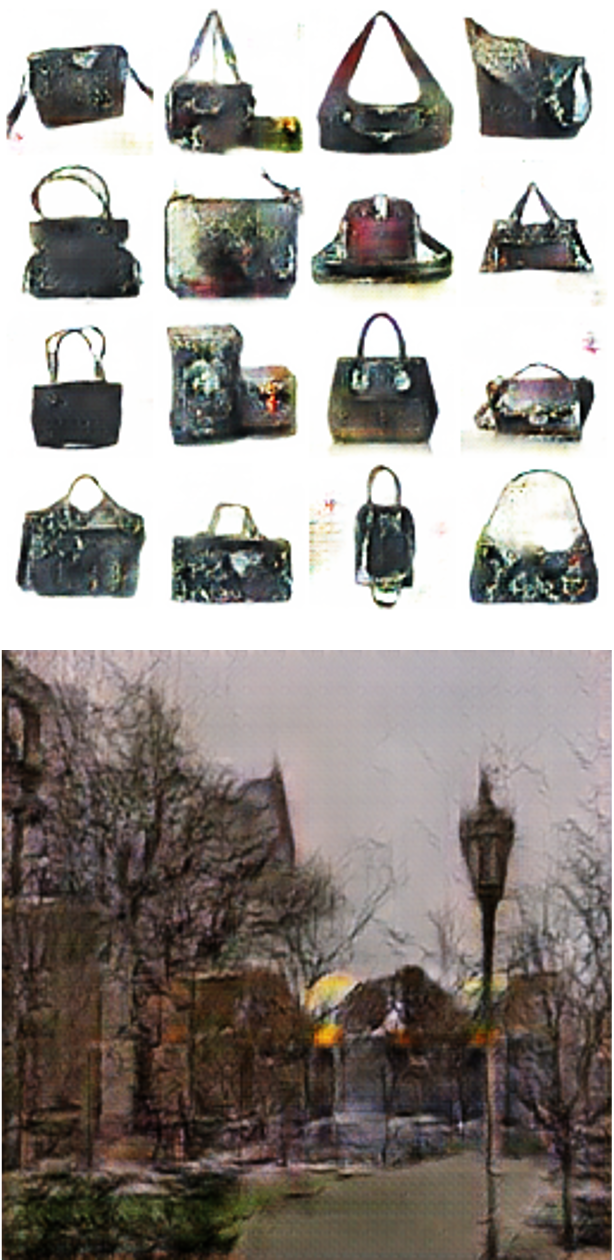}
    \end{subfigure} &
     \begin{subfigure}[t]{0.138\linewidth}
        \centering
        \includegraphics[width=\linewidth]{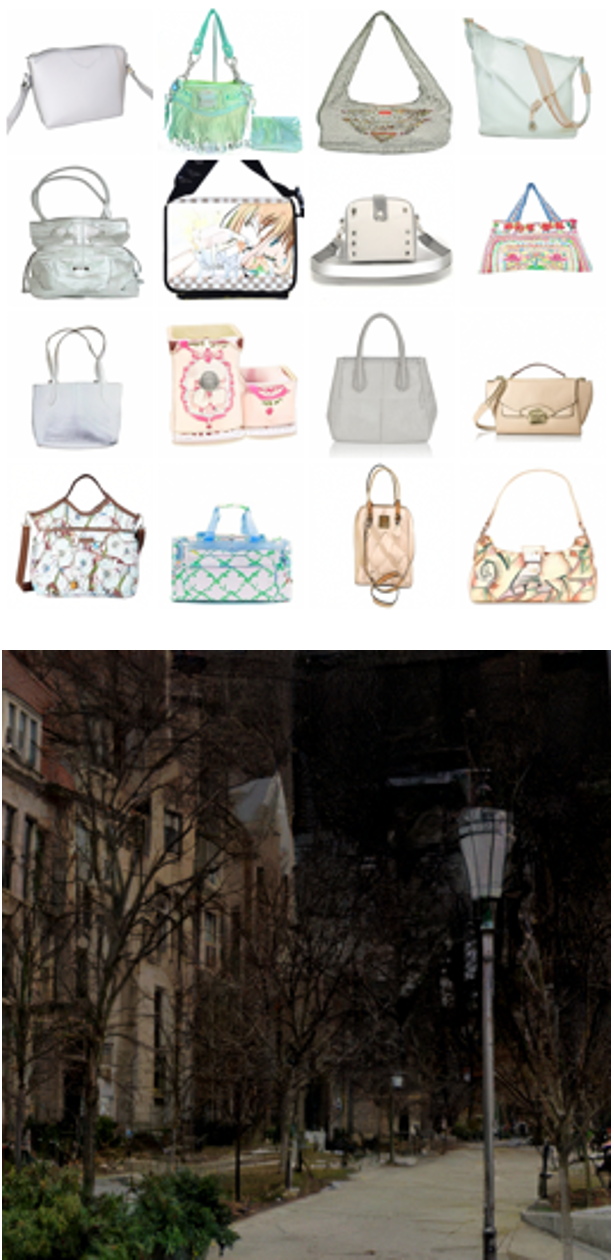}
    \end{subfigure} & 
     \begin{subfigure}[t]{0.138\linewidth}
        \centering
        \includegraphics[width=\linewidth]{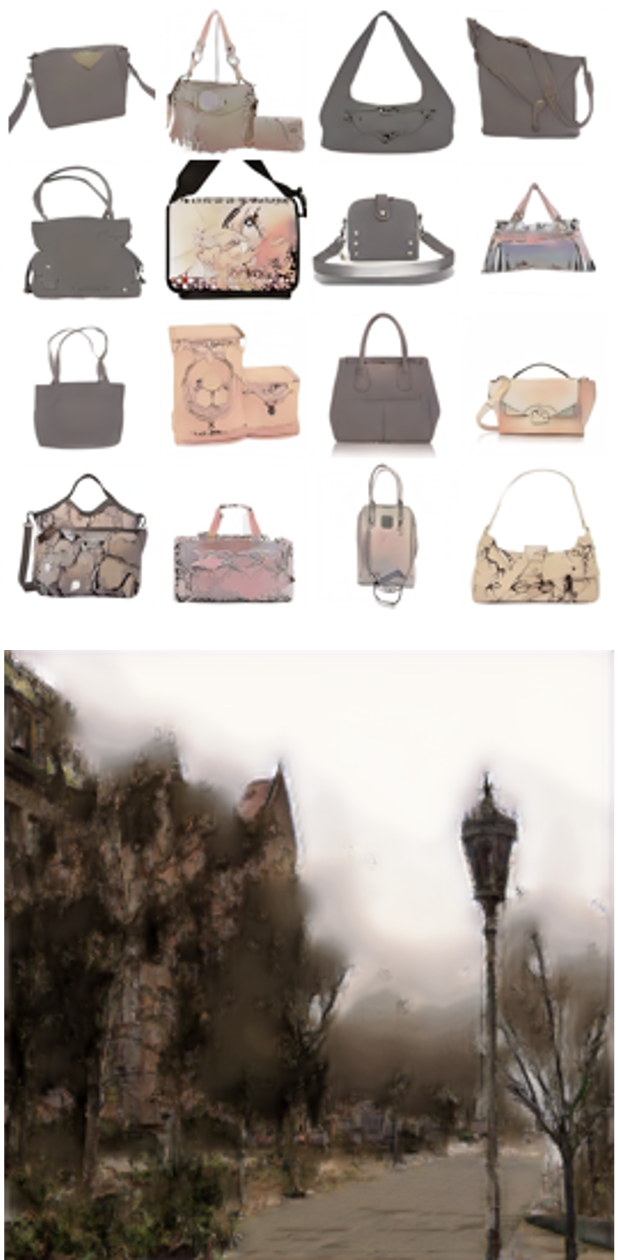}
    \end{subfigure} & 
     \begin{subfigure}[t]{0.138\linewidth}
        \centering
        \includegraphics[width=\linewidth]{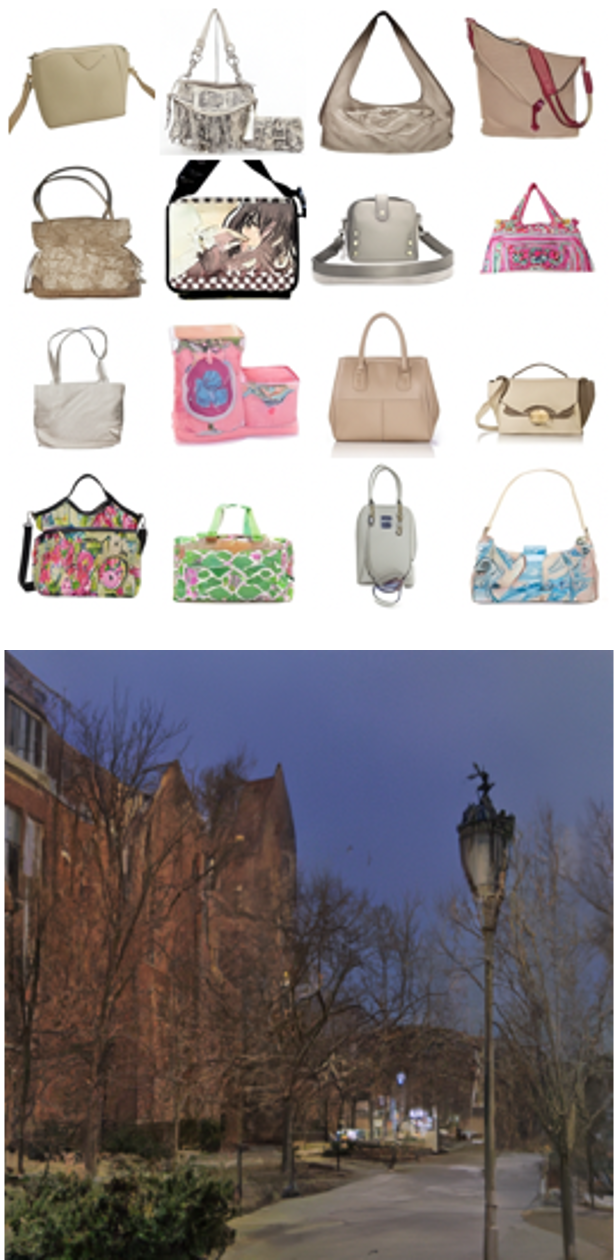}
    \end{subfigure} & 
      \begin{subfigure}[t]{0.138\linewidth}
        \centering
        \includegraphics[width=\linewidth]{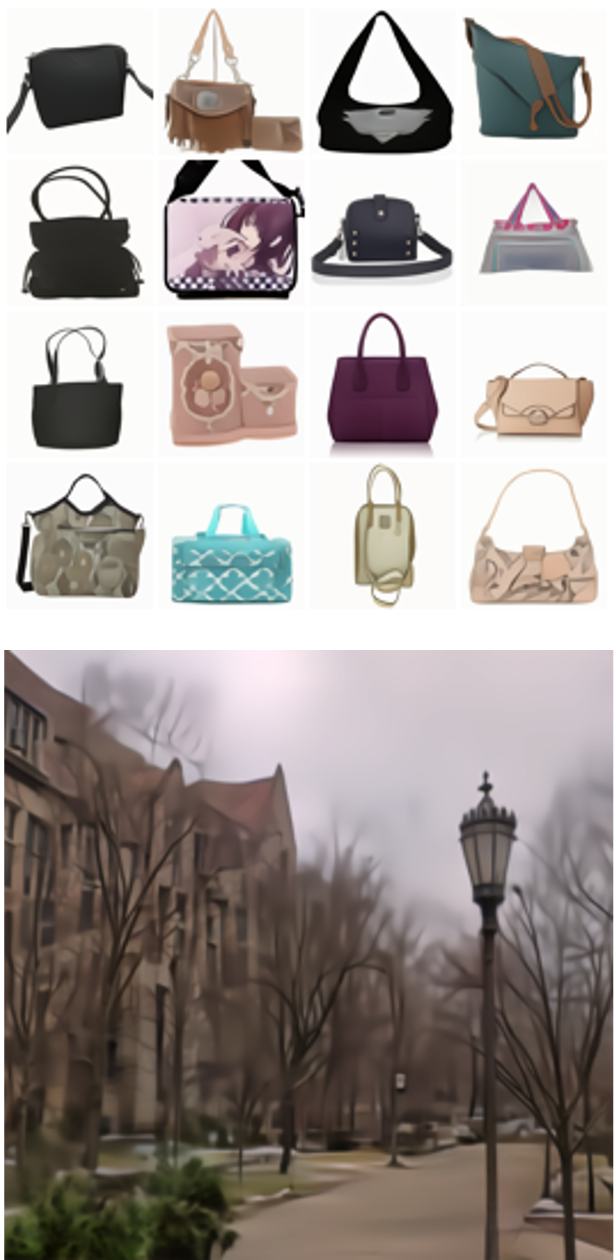}
    \end{subfigure} & 
     \begin{subfigure}[t]{0.138\linewidth}
        \centering
        \includegraphics[width=\linewidth]{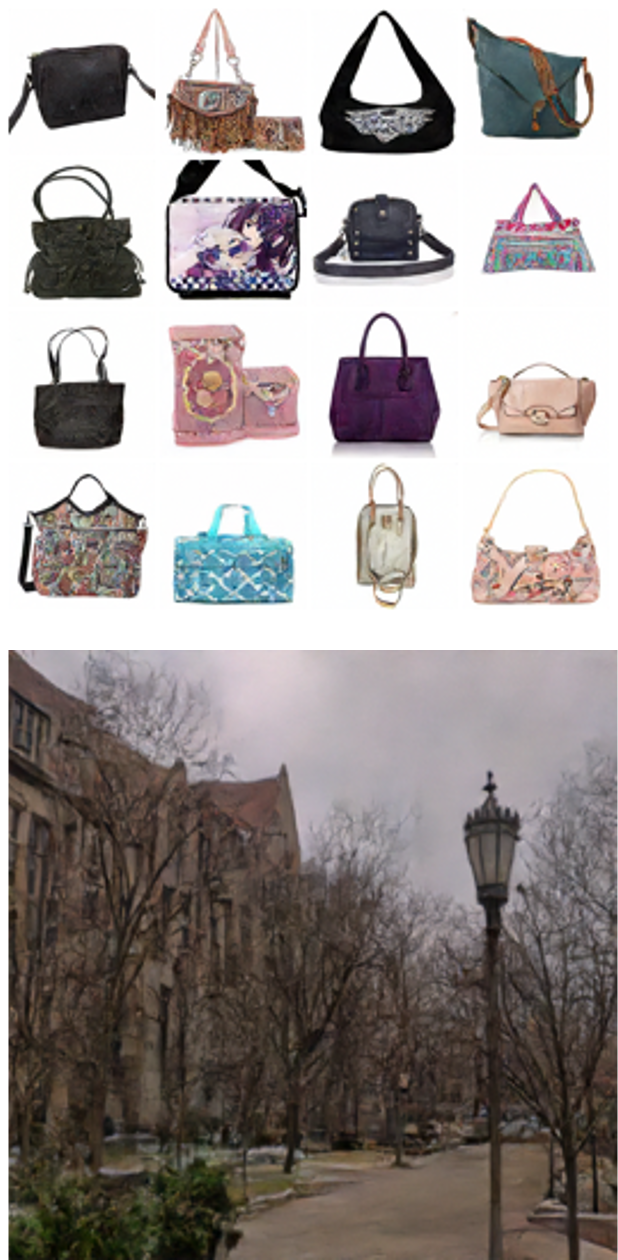}
    \end{subfigure} \\
    Condition & Pix2Pix~\scitenum{isola2017image} & SDEdit~\scitenum{meng2022sdedit} & \thead{Rectified\\ Flow}~\scitenum{liu2022flow}  & I$^2$SB~\scitenum{liu2023i2sb} & \thead{\textbf{DDBM (VE)}, \\ ODE sampler} &  \thead{\textbf{DDBM (VE)}, \\ hybrid sampler}
       
    \end{tabular}
    \vspace{-3mm}
    \caption{Qualitative comparison with the most relevant baselines.}
    \label{fig:e2h-diode}
    \vspace{-2mm}
\end{figure}

\subsection{Image-to-Image Translation}

% \begin{figure}
%     \centering
%     \begin{subfigure}[t]{0.45\linewidth}
%         \centering
%         \includegraphics[width=\linewidth]{figures/e2h.png}
%         \caption{Edges$\to$Handbags translation. Top: handbag sketches. Bottom: handbag coloring.}
%         \label{fig:e2h}
%     \end{subfigure}
%     \hspace{0.02\linewidth}
%     \begin{subfigure}[t]{0.45\linewidth}
%         \centering
%         \includegraphics[width=\linewidth]{figures/diode.png}
%         \caption{DIODE Normal$\to$RGB translation. Left: normal maps. Right: translation results.}
%         \label{fig:diode}
%     \end{subfigure}
%     \caption{Edges$\to$Handbags and DIODE Normal$\to$RGB translation results.}
%     \label{fig:e2h-diode}
% \end{figure}

We demonstrate that \model can deliver competitive results in general image-to-image translation tasks. We evaluate on datasets with different image resolutions to demonstrate its applicability on a variety of scales. We choose Edges$\to$Handbags~\citep{isola2017image}  scaled to $64\times 64$ pixels, which contains image pairs for translating from edge maps to colored handbags, and DIODE-Outdoor~\citep{diode_dataset} scaled to $256\times 256$, which contains normal maps and RGB images of real-world outdoor scenes. For evaluation metrics, we use Fr\'echet Inception Distance (FID)~\citep{heusel2017gans} and Inception Scores (IS)~\citep{barratt2018note} evaluated on all training samples translation quality, and we use LPIPS~\citep{zhang2018perceptual} and MSE (in $[-1, 1]$ scale) to measure perceptual similarity and translation faithfulness. 

We compare with Pix2Pix~\citep{isola2017image}, SDEdit~\citep{meng2022sdedit}, DDIB~\citep{su2022dual}, Rectified Flow~\citep{liu2022flow}, and I$^2$SB~\citep{liu2023i2sb} as they are built for image-to-image translation. For SDEdit we train unconditional EDM on the target domain, \eg colored images, and initialize the translation by noising source image, \eg sketches, and generate by EDM sampler given the noisy image. The other baseline methods are run with their respective repo while using the same network architecture as ours. Diffusion and transport-based methods are evaluated with the same number of function evaluations ($N=40$, which is the default for EDM sampler for $64\times 64$ images) to demonstrate our sampler's effectiveness in the regime when the number of sampling steps are low. Results are shown in \tabref{tab:pix-i2i} and additional settings are specified in \appref{sec:exp-details}.

 We observe that our model can perform translation with both high generation quality and faithfulness, and we find that VP bridges outperform VE bridges in some cases. In contrast, Rectified-Flow as an OT-based method struggles to perform well when the two domains share little low-level similarities (\eg color, hue). DDIB also fails to produce coherent translation due to the wide differences in pixel-space distribution between the paired data. I$^{2}$SB comes closest to our method, but falls short when limited by computational constraints, \ie NFE is low. We additionally show qualitative comparison with the most performant baselines in \figref{fig:e2h-diode}. More visual results can be found in \appref{sec:more-res}.

\subsection{Ablation Studies}

We now study the effect of our preconditioning and hybrid samplers on generation quality in the context of both VE and VP bridge (see \appref{sec:exp-details} for VP bridge parameterization). In the left column of \figref{fig:ablation}, we fix the guidance scale $w$ at 1 and vary the Euler step size $s$ from 0 to 0.9 to introduce stochasticity. We see a significant decrease in FID score as we increase $s$ which produces the best performance at some value between 0 and 1 (\eg $s=0.3$ for Edges$\rightarrow$Handbags). \figref{fig:e2h-diode} also shows that the ODE sampler (\ie $s=0$) produces blurry images while our hybrid sampler produces considerably sharper results. On the right column, we study the effect of $w$ (from 0 to 1) with fixed $s$. We observe that VE bridges are not affected by the change in $w$ whereas VP bridges heavily rely on setting $w=1$. We hypothesize that this is due to the fact that VP bridges follow "curved paths" and destroy signals in between, so it is reliant on Doob's $h$-function for further guidance towards correct probability distribution.

We also study the effect of our preconditioning in \tabref{tab:ablation}. Our baseline without our preconditioning and our sampler is a simple model that directly matches output of the neural network to the training target and generates using EDM~\citep{karras2022elucidating} sampler. We see that each introduced component further boosts the generation performance. Therefore, we can conclude that the introduced practical components are essential for the success of our \model.

% \begin{figure}[t]
%     \centering
%     \includegraphics[width=\linewidth]{figures/d2n.png}
%     \caption{Visualization for Day$\to$Night translation. Top: day images. Bottom: night translations.}
%     \label{fig:d2n}
% \end{figure}

\begin{table}[t]

\centering

\resizebox{0.98\linewidth}{!}{
\begin{tabular}{lcccccccc}
\toprule
  &   \multicolumn{4}{c}{Edges$\to$Handbags-64$\times$64}  & \multicolumn{4}{c}{DIODE-256$\times$256}  \\
\cmidrule(r){2-5} \cmidrule(l){6-9}
& FID $\downarrow$ & IS $\uparrow$ & LPIPS $\downarrow$ & MSE $\downarrow$ & FID $\downarrow$ & IS $\uparrow$ & LPIPS $\downarrow$& MSE $\downarrow$  \\
\midrule
Pix2Pix~\citep{isola2017image} & 74.8 & 4.24 & 0.356 & 0.209 & 82.4 & 4.22 & 0.556 &  0.133 \\
\midrule
DDIB~\citep{su2022dual} &186.84  &   2.04& 0.869 & 1.05 & 242.3&  4.22 & 0.798 & 0.794 \\
\midrule
SDEdit~\citep{meng2022sdedit} & 26.5 &  \textbf{3.58} & 0.271 & 0.510 & 31.14 &  5.70& 0.714  &  0.534\\
\midrule
Rectified Flow~\citep{liu2022flow} &  25.3& 2.80 & 0.241 & 0.088 & 77.18 & 5.87 & 0.534 &0.157 \\
\midrule
I$^{2}$SB~\citep{liu2023i2sb} & 7.43 & 3.40 & 0.244  &0.191& 9.34 & 5.77 & 0.373 & 0.145 \\
\midrule
% \midrule
% \textbf{\model}-Det & 14.37 &  &  & 58.47 & 3.70 & 0.436 & & &  \\
\midrule
\textbf{\model (VE)} &  2.93 & 3.58  & \textbf{0.131} &  \textbf{0.013} & 8.51 & 6.03 &  \textbf{0.226} & \textbf{0.0107} \\
\textbf{\model (VP)} & \textbf{1.83} &  \textbf{3.73} & 0.142 & 0.0402 & \textbf{4.43} & \textbf{6.21} &  0.244 & 0.0839 \\
\bottomrule
\end{tabular}
}
\caption{Quantitative evaluation of pixel-space image-to-image translation.}
\label{tab:pix-i2i}
\vspace{-5mm}
\end{table}

\begin{figure}[h]
    \vspace{-5mm}
    \centering
    \parbox{0.49\linewidth}{
    \begin{tabular}{c}
        \begin{subfigure}[t]{\linewidth}
        \centering
        \caption{Edges$\rightarrow$Handbags}
        \vspace{1mm}
        \includegraphics[width=\linewidth]{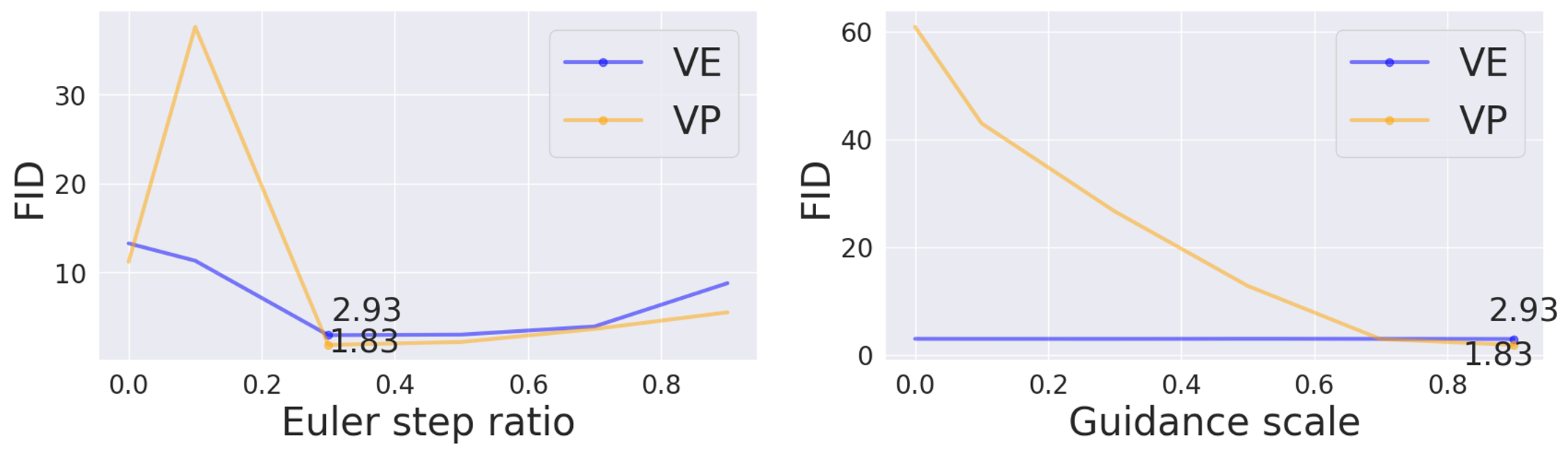}
    \end{subfigure} \\
          \begin{subfigure}[t]{\linewidth}
        \centering
        \includegraphics[width=\linewidth]{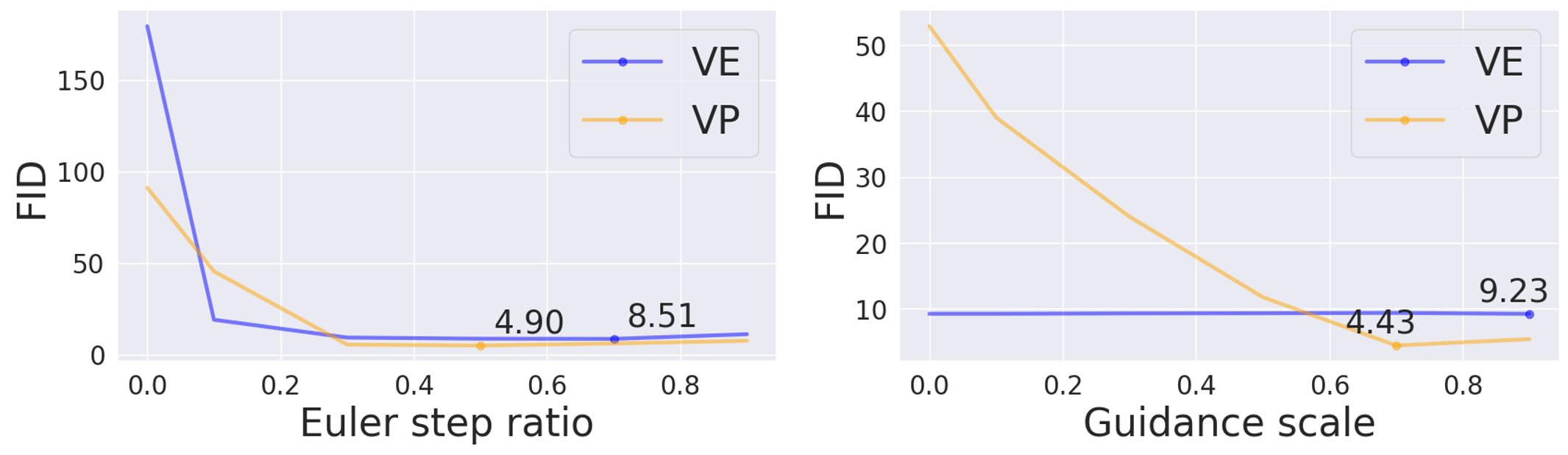}
        \vspace{-5mm}
        \caption{DIODE}
    \end{subfigure}
    \end{tabular}
    \caption{Ablation studies on Euler step ratio $s$ and guidance scale $w$: $w=1$ for all ablation on $s$ and $s$ is set to the best-performing value for each dataset for ablation on $w$. }
    \label{fig:ablation}
    }
    \hfill
\parbox{.48\linewidth}{
\resizebox{\linewidth}{!}{
\begin{tabular}{cccccc}
\toprule
 \multirow{2}{*}{\thead{Our\\ precond.}} & \multirow{2}{*}{\thead{Our\\ sampler}} &   \multicolumn{2}{c}{E$\to$H-64$\times$64}  & \multicolumn{2}{c}{DIODE-256$\times$256}  \\
\cmidrule(r){3-4} \cmidrule(l){5-6}
& & VE & VP &   VE & VP  \\
\midrule
\xmark & \xmark  & 14.02 &  11.76 & 126.3 & 96.93  \\
\midrule
\cmark & \xmark  & 13.26 & 11.19 & 79.25 & 91.07 \\
\midrule
\xmark & \cmark    & 13.11 & 29.91  &  91.31 & 21.92 \\
\midrule
\cmark & \cmark   & \textbf{2.93} & \textbf{1.83}  &  \textbf{8.51} & \textbf{4.43} \\
\bottomrule
\end{tabular}}
\captionof{table}{Ablation study on the effect of sampler and preconditioning on FID. Cross mark on our preconditioning means no output reparameterization and directly use network output to match training target. Cross mark on our sampler means we reuse the ODE sampler from EDM with the same setting. E$\rightarrow$H is a short-hand for Edges$\rightarrow$Handbags.}
\label{tab:ablation}
}
\end{figure}

\begin{figure}[t]
% \vspace{-5mm}
\centering
\parbox{0.51\linewidth}{
% \vspace{-4mm}
\begin{subfigure}[t]{\linewidth}
     \includegraphics[width=\linewidth]{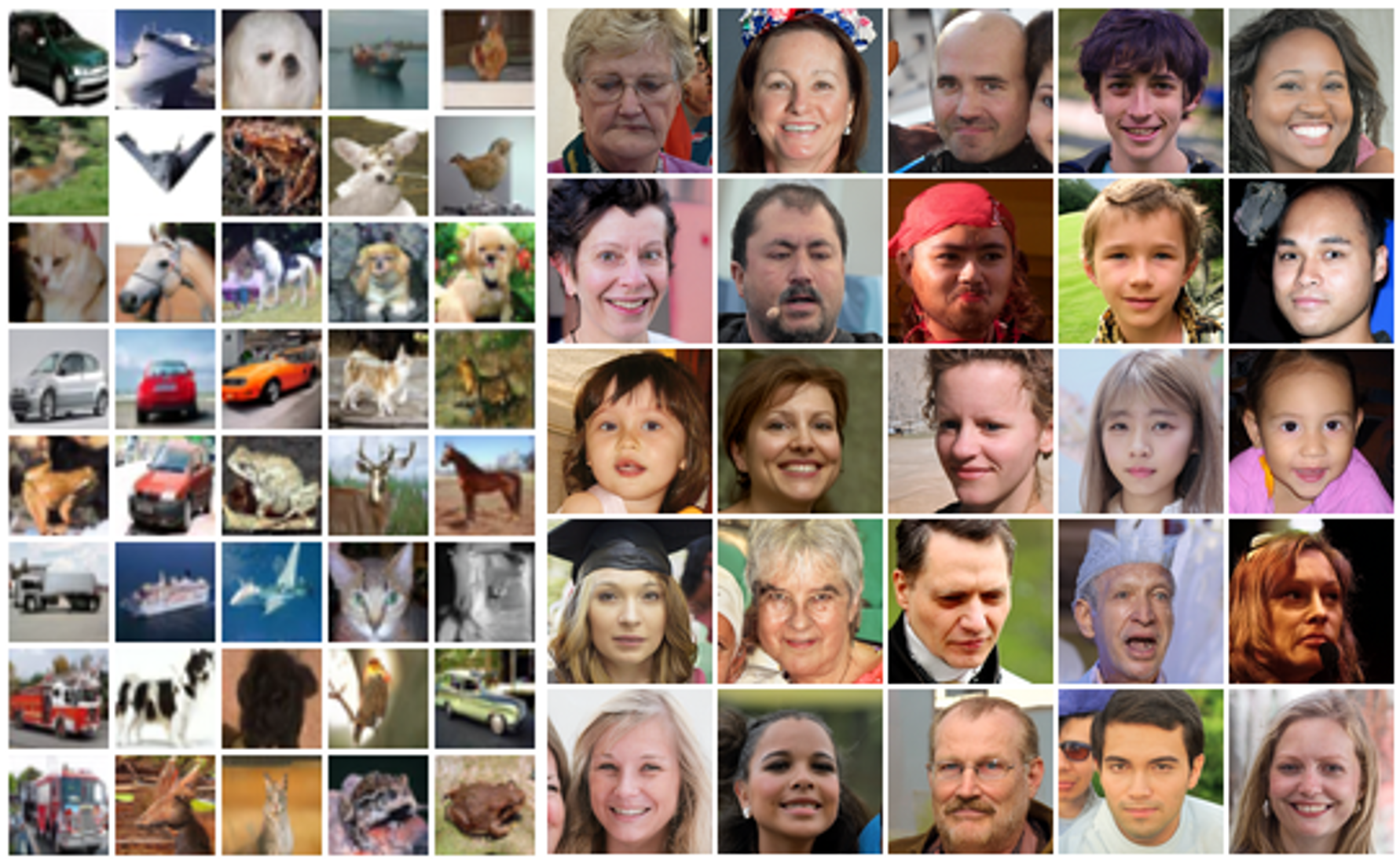}
\end{subfigure}
\caption{Generation on CIFAR-10 and FFHQ-$64\times 64$.}
    \label{fig:gen}
}
\hfill
\parbox{.47\linewidth}{
% \capbtabbox{
\vspace{-2mm}
\resizebox{\linewidth}{!}{
\begin{tabular}{lcccc}
\toprule
  &   \multicolumn{2}{c}{CIFAR-10} & \multicolumn{2}{c}{FFHQ-$64\times64$ }\\
  \cmidrule{2-5}
  &NFE $\downarrow$ & FID $\downarrow$ & NFE $\downarrow$& FID $\downarrow$\\
\midrule
DDPM~\scitenum{ho2020denoising} & 1000 & 3.17 & 1000 & 3.52 \\
\midrule
DDIM~\scitenum{song2020denoising} & 50 & 4.67 & 50 & 5.18 \\
\midrule
DDPM++~\scitenum{song2020score} & 1000 & 3.01 & 1000 & 3.39 \\
\midrule
NCSN++~\scitenum{song2020score} & 1000 & 3.77 & 1000 & 25.95 \\
\midrule
Rectified Flow~\scitenum{liu2022flow} & 127 & 2.58 & 152 &  4.45    \\
\midrule
EDM~\scitenum{karras2022elucidating} & 35 & \textbf{2.04} & 79 & 2.53 \\
\midrule
\midrule
\textbf{\model} & 35 & 2.06 & 79 & \textbf{2.44}   \\
\bottomrule
\end{tabular}}
\captionof{table}{Evaluation of unconditional generation.}
\label{tab:genresults}
% }
}

\vspace{-4mm}
\end{figure}

\subsection{Unconditional Generation}

When one side of the distribution becomes Gaussian distribution, our framework exactly reduces to that of diffusion models. Specifically, during training when the end point $\rvx_T\sim\gN(\alpha_T\rvx_0, \sigma_T^2\mI)$, our intermediate bridge samples $\rvx_t$ follows the distribution $\rvx_t\sim\gN(\alpha_t\rvx_0, \sigma_t^2\mI)$.  We empirically verify that using our bridge sampling and the pred-$\rvx$ objective inspired by EDM, we can recover its performance by using our more generalized parameterization. 

We evaluate our method on CIFAR-10~\citep{krizhevsky2009learning} and FFHQ-$64\times 64$~\citep{karras2019style} which are processed according to \citet{karras2022elucidating}. We use FID score for quantitative evaluation using 50K generated images and use number of function evaluations (NFE) for generation efficiency. 
We compare our generation results against diffusion-based and optimal transport-based models including DDPM~\citep{ho2020denoising}, 
DDIM~\citep{song2020denoising},
DDPM++~\citep{song2020score}, 
NCSN++~\citep{song2020score}, Rectified Flow~\citep{liu2022flow}, EDM~\citep{karras2022elucidating}. Quantitative results are presented in \tabref{tab:genresults} and generated samples are shown in \figref{fig:gen}.

We observe that our model is able to match EDM performance with negligible degradation in FID scores for CIFAR-10 and marginal improvement for FFHQ-$64\times 64$. This corroborates our claim that our method can benefit from advances in diffusion models and generalize many of the advanced parameterization techniques such as those introduced in EDM. 
\section{Conclusion}

In this work, we introduce Denoising Diffusion Bridge Models, a novel class of models that builds a stochastic bridge between paired samples with tractable marginal distributions in between. 
%Such a bridge is constructed using a diffusion process conditioned on a known endpoint using Doob's $h$-transform. 
The model is learned by matching the conditional score of a tractable bridge distribution, which allows one to transport from one distribution to another via a new reverse SDE or probability flow ODE. Additionally, this generalized framework shares many similarities with diffusion models, thus allowing us to reuse and generalize many designs of diffusion models.
%including loss reparameterization, network pre-conditioning, and fast ODE samplers. 
We believe that \model is a significant contribution towards a general framework for distribution translation. In the era of generative AI, \model has a further role to play.

\clearpage
{\small
\bibliographystyle{plainnat}
\bibliography{egbib}

\begin{thebibliography}{55}
\providecommand{\natexlab}[1]{#1}
\providecommand{\url}[1]{\texttt{#1}}
\expandafter\ifx\csname urlstyle\endcsname\relax
  \providecommand{\doi}[1]{doi: #1}\else
  \providecommand{\doi}{doi: \begingroup \urlstyle{rm}\Url}\fi

\bibitem[Albergo et~al.(2023)Albergo, Boffi, and Vanden-Eijnden]{albergo2023stochastic}
Michael~S Albergo, Nicholas~M Boffi, and Eric Vanden-Eijnden.
\newblock Stochastic interpolants: A unifying framework for flows and diffusions.
\newblock \emph{arXiv preprint arXiv:2303.08797}, 2023.

\bibitem[Albergo and Vanden-Eijnden(2023)]{albergo2023building}
Michael~Samuel Albergo and Eric Vanden-Eijnden.
\newblock Building normalizing flows with stochastic interpolants.
\newblock In \emph{The Eleventh International Conference on Learning Representations}, 2023.
\newblock URL \url{https://openreview.net/forum?id=li7qeBbCR1t}.

\bibitem[Barratt and Sharma(2018)]{barratt2018note}
Shane Barratt and Rishi Sharma.
\newblock A note on the inception score.
\newblock \emph{arXiv preprint arXiv:1801.01973}, 2018.

\bibitem[De~Bortoli et~al.(2021)De~Bortoli, Thornton, Heng, and Doucet]{de2021diffusion}
Valentin De~Bortoli, James Thornton, Jeremy Heng, and Arnaud Doucet.
\newblock Diffusion schr{\"o}dinger bridge with applications to score-based generative modeling.
\newblock \emph{Advances in Neural Information Processing Systems}, 34:\penalty0 17695--17709, 2021.

\bibitem[Delbracio and Milanfar(2023)]{delbracio2023inversion}
Mauricio Delbracio and Peyman Milanfar.
\newblock Inversion by direct iteration: An alternative to denoising diffusion for image restoration.
\newblock \emph{arXiv preprint arXiv:2303.11435}, 2023.

\bibitem[Delyon and Hu(2006)]{delyon2006simulation}
Bernard Delyon and Ying Hu.
\newblock Simulation of conditioned diffusion and application to parameter estimation.
\newblock \emph{Stochastic Processes and their Applications}, 116\penalty0 (11):\penalty0 1660--1675, 2006.

\bibitem[Dhariwal and Nichol(2021)]{dhariwal2021diffusion}
Prafulla Dhariwal and Alexander Nichol.
\newblock Diffusion models beat gans on image synthesis.
\newblock \emph{Advances in Neural Information Processing Systems}, 34:\penalty0 8780--8794, 2021.

\bibitem[Doob and Doob(1984)]{doob1984classical}
Joseph~L Doob and JI~Doob.
\newblock \emph{Classical potential theory and its probabilistic counterpart}, volume 262.
\newblock Springer, 1984.

\bibitem[Goodfellow et~al.(2014)Goodfellow, Pouget-Abadie, Mirza, Xu, Warde-Farley, Ozair, Courville, and Bengio]{Goodfellow2014GenerativeAN}
Ian~J. Goodfellow, Jean Pouget-Abadie, Mehdi Mirza, Bing Xu, David Warde-Farley, Sherjil Ozair, Aaron~C. Courville, and Yoshua Bengio.
\newblock Generative adversarial nets.
\newblock In \emph{NIPS}, 2014.

\bibitem[Heng et~al.(2021)Heng, De~Bortoli, Doucet, and Thornton]{heng2021simulating}
Jeremy Heng, Valentin De~Bortoli, Arnaud Doucet, and James Thornton.
\newblock Simulating diffusion bridges with score matching.
\newblock \emph{arXiv preprint arXiv:2111.07243}, 2021.

\bibitem[Heusel et~al.(2017)Heusel, Ramsauer, Unterthiner, Nessler, and Hochreiter]{heusel2017gans}
Martin Heusel, Hubert Ramsauer, Thomas Unterthiner, Bernhard Nessler, and Sepp Hochreiter.
\newblock Gans trained by a two time-scale update rule converge to a local nash equilibrium.
\newblock \emph{Advances in neural information processing systems}, 30, 2017.

\bibitem[Ho and Salimans(2022)]{ho2022classifier}
Jonathan Ho and Tim Salimans.
\newblock Classifier-free diffusion guidance.
\newblock \emph{arXiv preprint arXiv:2207.12598}, 2022.

\bibitem[Ho et~al.(2020)Ho, Jain, and Abbeel]{ho2020denoising}
Jonathan Ho, Ajay Jain, and Pieter Abbeel.
\newblock Denoising diffusion probabilistic models.
\newblock \emph{Advances in Neural Information Processing Systems}, 33:\penalty0 6840--6851, 2020.

\bibitem[Ho et~al.(2022)Ho, Chan, Saharia, Whang, Gao, Gritsenko, Kingma, Poole, Norouzi, Fleet, et~al.]{ho2022imagen}
Jonathan Ho, William Chan, Chitwan Saharia, Jay Whang, Ruiqi Gao, Alexey Gritsenko, Diederik~P Kingma, Ben Poole, Mohammad Norouzi, David~J Fleet, et~al.
\newblock Imagen video: High definition video generation with diffusion models.
\newblock \emph{arXiv preprint arXiv:2210.02303}, 2022.

\bibitem[Hoogeboom et~al.(2023)Hoogeboom, Heek, and Salimans]{hoogeboom2023simple}
Emiel Hoogeboom, Jonathan Heek, and Tim Salimans.
\newblock simple diffusion: End-to-end diffusion for high resolution images.
\newblock \emph{arXiv preprint arXiv:2301.11093}, 2023.

\bibitem[Isola et~al.(2017)Isola, Zhu, Zhou, and Efros]{isola2017image}
Phillip Isola, Jun-Yan Zhu, Tinghui Zhou, and Alexei~A Efros.
\newblock Image-to-image translation with conditional adversarial networks.
\newblock In \emph{Proceedings of the IEEE conference on computer vision and pattern recognition}, pages 1125--1134, 2017.

\bibitem[Karras et~al.(2019)Karras, Laine, and Aila]{karras2019style}
Tero Karras, Samuli Laine, and Timo Aila.
\newblock A style-based generator architecture for generative adversarial networks.
\newblock In \emph{Proceedings of the IEEE/CVF conference on computer vision and pattern recognition}, pages 4401--4410, 2019.

\bibitem[Karras et~al.(2022)Karras, Aittala, Aila, and Laine]{karras2022elucidating}
Tero Karras, Miika Aittala, Timo Aila, and Samuli Laine.
\newblock Elucidating the design space of diffusion-based generative models.
\newblock \emph{arXiv preprint arXiv:2206.00364}, 2022.

\bibitem[Kingma et~al.(2021)Kingma, Salimans, Poole, and Ho]{kingma2021variational}
Diederik Kingma, Tim Salimans, Ben Poole, and Jonathan Ho.
\newblock Variational diffusion models.
\newblock \emph{Advances in neural information processing systems}, 34:\penalty0 21696--21707, 2021.

\bibitem[Krizhevsky et~al.(2009)Krizhevsky, Hinton, et~al.]{krizhevsky2009learning}
Alex Krizhevsky, Geoffrey Hinton, et~al.
\newblock Learning multiple layers of features from tiny images.
\newblock 2009.

\bibitem[Li et~al.(2023)Li, Xue, Liu, and Lai]{li2023bbdm}
Bo~Li, Kaitao Xue, Bin Liu, and Yu-Kun Lai.
\newblock Bbdm: Image-to-image translation with brownian bridge diffusion models.
\newblock In \emph{Proceedings of the IEEE/CVF Conference on Computer Vision and Pattern Recognition}, pages 1952--1961, 2023.

\bibitem[Lipman et~al.(2023)Lipman, Chen, Ben-Hamu, Nickel, and Le]{lipman2023flow}
Yaron Lipman, Ricky T.~Q. Chen, Heli Ben-Hamu, Maximilian Nickel, and Matthew Le.
\newblock Flow matching for generative modeling.
\newblock In \emph{The Eleventh International Conference on Learning Representations}, 2023.
\newblock URL \url{https://openreview.net/forum?id=PqvMRDCJT9t}.

\bibitem[Liu et~al.(2023)Liu, Vahdat, Huang, Theodorou, Nie, and Anandkumar]{liu2023i2sb}
Guan-Horng Liu, Arash Vahdat, De-An Huang, Evangelos~A Theodorou, Weili Nie, and Anima Anandkumar.
\newblock I$^2$sb: Image-to-image schr{\"o}dinger bridge.
\newblock \emph{arXiv}, 2023.

\bibitem[Liu et~al.(2022{\natexlab{a}})Liu, Gong, and Liu]{liu2022flow}
Xingchao Liu, Chengyue Gong, and Qiang Liu.
\newblock Flow straight and fast: Learning to generate and transfer data with rectified flow.
\newblock \emph{arXiv preprint arXiv:2209.03003}, 2022{\natexlab{a}}.

\bibitem[Liu et~al.(2022{\natexlab{b}})Liu, Wu, Ye, and Liu]{liu2022let}
Xingchao Liu, Lemeng Wu, Mao Ye, and Qiang Liu.
\newblock Let us build bridges: Understanding and extending diffusion generative models.
\newblock \emph{arXiv preprint arXiv:2208.14699}, 2022{\natexlab{b}}.

\bibitem[Lu et~al.(2022{\natexlab{a}})Lu, Zhou, Bao, Chen, Li, and Zhu]{lu2022dpm}
Cheng Lu, Yuhao Zhou, Fan Bao, Jianfei Chen, Chongxuan Li, and Jun Zhu.
\newblock Dpm-solver: A fast ode solver for diffusion probabilistic model sampling in around 10 steps.
\newblock \emph{arXiv preprint arXiv:2206.00927}, 2022{\natexlab{a}}.

\bibitem[Lu et~al.(2022{\natexlab{b}})Lu, Zhou, Bao, Chen, Li, and Zhu]{lu2022dpm2}
Cheng Lu, Yuhao Zhou, Fan Bao, Jianfei Chen, Chongxuan Li, and Jun Zhu.
\newblock Dpm-solver++: Fast solver for guided sampling of diffusion probabilistic models.
\newblock \emph{arXiv preprint arXiv:2211.01095}, 2022{\natexlab{b}}.

\bibitem[Meng et~al.(2022)Meng, He, Song, Song, Wu, Zhu, and Ermon]{meng2022sdedit}
Chenlin Meng, Yutong He, Yang Song, Jiaming Song, Jiajun Wu, Jun-Yan Zhu, and Stefano Ermon.
\newblock {SDE}dit: Guided image synthesis and editing with stochastic differential equations.
\newblock In \emph{International Conference on Learning Representations}, 2022.

\bibitem[Nichol and Dhariwal(2021)]{nichol2021improved}
Alexander~Quinn Nichol and Prafulla Dhariwal.
\newblock Improved denoising diffusion probabilistic models.
\newblock In \emph{International Conference on Machine Learning}, pages 8162--8171. PMLR, 2021.

\bibitem[Peebles and Xie(2023)]{peebles2023scalable}
William Peebles and Saining Xie.
\newblock Scalable diffusion models with transformers.
\newblock In \emph{Proceedings of the IEEE/CVF International Conference on Computer Vision}, pages 4195--4205, 2023.

\bibitem[Peluchetti()]{peluchettinon}
Stefano Peluchetti.
\newblock Non-denoising forward-time diffusions.

\bibitem[Peluchetti(2023)]{peluchetti2023diffusion}
Stefano Peluchetti.
\newblock Diffusion bridge mixture transports, schr$\backslash$" odinger bridge problems and generative modeling.
\newblock \emph{arXiv preprint arXiv:2304.00917}, 2023.

\bibitem[Pooladian et~al.(2023)Pooladian, Ben-Hamu, Domingo-Enrich, Amos, Lipman, and Chen]{pooladian2023multisample}
Aram-Alexandre Pooladian, Heli Ben-Hamu, Carles Domingo-Enrich, Brandon Amos, Yaron Lipman, and Ricky Chen.
\newblock Multisample flow matching: Straightening flows with minibatch couplings.
\newblock \emph{arXiv preprint arXiv:2304.14772}, 2023.

\bibitem[Ramesh et~al.(2022)Ramesh, Dhariwal, Nichol, Chu, and Chen]{Ramesh2022HierarchicalTI}
Aditya Ramesh, Prafulla Dhariwal, Alex Nichol, Casey Chu, and Mark Chen.
\newblock Hierarchical text-conditional image generation with clip latents.
\newblock \emph{ArXiv}, abs/2204.06125, 2022.

\bibitem[Rogers and Williams(2000)]{rogers2000diffusions}
L~Chris~G Rogers and David Williams.
\newblock \emph{Diffusions, Markov processes and martingales: Volume 2, It{\^o} calculus}, volume~2.
\newblock Cambridge university press, 2000.

\bibitem[Rombach et~al.(2022)Rombach, Blattmann, Lorenz, Esser, and Ommer]{rombach2022high}
Robin Rombach, Andreas Blattmann, Dominik Lorenz, Patrick Esser, and Bj{\"o}rn Ommer.
\newblock High-resolution image synthesis with latent diffusion models.
\newblock In \emph{Proceedings of the IEEE/CVF Conference on Computer Vision and Pattern Recognition}, pages 10684--10695, 2022.

\bibitem[Saharia et~al.(2021)Saharia, Chan, Chang, Lee, Ho, Salimans, Fleet, and Norouzi]{Saharia2021PaletteID}
Chitwan Saharia, William Chan, Huiwen Chang, Chris~A. Lee, Jonathan Ho, Tim Salimans, David~J. Fleet, and Mohammad Norouzi.
\newblock Palette: Image-to-image diffusion models.
\newblock \emph{ACM SIGGRAPH 2022 Conference Proceedings}, 2021.

\bibitem[Salimans and Ho(2022)]{salimans2022progressive}
Tim Salimans and Jonathan Ho.
\newblock Progressive distillation for fast sampling of diffusion models.
\newblock \emph{arXiv preprint arXiv:2202.00512}, 2022.

\bibitem[S{\"a}rkk{\"a} and Solin(2019)]{sarkka2019applied}
Simo S{\"a}rkk{\"a} and Arno Solin.
\newblock \emph{Applied stochastic differential equations}, volume~10.
\newblock Cambridge University Press, 2019.

\bibitem[Schauer et~al.(2017)Schauer, Van Der~Meulen, and Van~Zanten]{schauer2017guided}
Moritz Schauer, Frank Van Der~Meulen, and Harry Van~Zanten.
\newblock Guided proposals for simulating multi-dimensional diffusion bridges.
\newblock 2017.

\bibitem[Shi et~al.(2023)Shi, De~Bortoli, Campbell, and Doucet]{shi2023diffusion}
Yuyang Shi, Valentin De~Bortoli, Andrew Campbell, and Arnaud Doucet.
\newblock Diffusion schr$\backslash$" odinger bridge matching.
\newblock \emph{arXiv preprint arXiv:2303.16852}, 2023.

\bibitem[Sohl-Dickstein et~al.(2015)Sohl-Dickstein, Weiss, Maheswaranathan, and Ganguli]{sohl2015deep}
Jascha Sohl-Dickstein, Eric Weiss, Niru Maheswaranathan, and Surya Ganguli.
\newblock Deep unsupervised learning using nonequilibrium thermodynamics.
\newblock In \emph{International Conference on Machine Learning}, pages 2256--2265. PMLR, 2015.

\bibitem[Somnath et~al.(2023)Somnath, Pariset, Hsieh, Martinez, Krause, and Bunne]{somnath2023aligned}
Vignesh~Ram Somnath, Matteo Pariset, Ya-Ping Hsieh, Maria~Rodriguez Martinez, Andreas Krause, and Charlotte Bunne.
\newblock Aligned diffusion schr$\backslash$" odinger bridges.
\newblock \emph{arXiv preprint arXiv:2302.11419}, 2023.

\bibitem[Song et~al.(2020{\natexlab{a}})Song, Meng, and Ermon]{song2020denoising}
Jiaming Song, Chenlin Meng, and Stefano Ermon.
\newblock Denoising diffusion implicit models.
\newblock \emph{arXiv preprint arXiv:2010.02502}, 2020{\natexlab{a}}.

\bibitem[Song and Ermon(2019)]{song2019generative}
Yang Song and Stefano Ermon.
\newblock Generative modeling by estimating gradients of the data distribution.
\newblock \emph{Advances in neural information processing systems}, 32, 2019.

\bibitem[Song et~al.(2020{\natexlab{b}})Song, Sohl-Dickstein, Kingma, Kumar, Ermon, and Poole]{song2020score}
Yang Song, Jascha Sohl-Dickstein, Diederik~P Kingma, Abhishek Kumar, Stefano Ermon, and Ben Poole.
\newblock Score-based generative modeling through stochastic differential equations.
\newblock \emph{arXiv preprint arXiv:2011.13456}, 2020{\natexlab{b}}.

\bibitem[Su et~al.(2022)Su, Song, Meng, and Ermon]{su2022dual}
Xuan Su, Jiaming Song, Chenlin Meng, and Stefano Ermon.
\newblock Dual diffusion implicit bridges for image-to-image translation.
\newblock In \emph{The Eleventh International Conference on Learning Representations}, 2022.

\bibitem[Szavits-Nossan and Evans(2015)]{szavits2015inequivalence}
Juraj Szavits-Nossan and Martin~R Evans.
\newblock Inequivalence of nonequilibrium path ensembles: the example of stochastic bridges.
\newblock \emph{Journal of Statistical Mechanics: Theory and Experiment}, 2015\penalty0 (12):\penalty0 P12008, 2015.

\bibitem[Tong et~al.(2023{\natexlab{a}})Tong, Malkin, Fatras, Atanackovic, Zhang, Huguet, Wolf, and Bengio]{tong2023simulation}
Alexander Tong, Nikolay Malkin, Kilian Fatras, Lazar Atanackovic, Yanlei Zhang, Guillaume Huguet, Guy Wolf, and Yoshua Bengio.
\newblock Simulation-free schr$\backslash$" odinger bridges via score and flow matching.
\newblock \emph{arXiv preprint arXiv:2307.03672}, 2023{\natexlab{a}}.

\bibitem[Tong et~al.(2023{\natexlab{b}})Tong, Malkin, Huguet, Zhang, Rector-Brooks, Fatras, Wolf, and Bengio]{tong2023improving}
Alexander Tong, Nikolay Malkin, Guillaume Huguet, Yanlei Zhang, Jarrid Rector-Brooks, Kilian Fatras, Guy Wolf, and Yoshua Bengio.
\newblock Improving and generalizing flow-based generative models with minibatch optimal transport.
\newblock In \emph{ICML Workshop on New Frontiers in Learning, Control, and Dynamical Systems}, 2023{\natexlab{b}}.

\bibitem[Vasiljevic et~al.(2019)Vasiljevic, Kolkin, Zhang, Luo, Wang, Dai, Daniele, Mostajabi, Basart, Walter, and Shakhnarovich]{diode_dataset}
Igor Vasiljevic, Nick Kolkin, Shanyi Zhang, Ruotian Luo, Haochen Wang, Falcon~Z. Dai, Andrea~F. Daniele, Mohammadreza Mostajabi, Steven Basart, Matthew~R. Walter, and Gregory Shakhnarovich.
\newblock {DIODE}: {A} {D}ense {I}ndoor and {O}utdoor {DE}pth {D}ataset.
\newblock \emph{CoRR}, abs/1908.00463, 2019.
\newblock URL \url{http://arxiv.org/abs/1908.00463}.

\bibitem[Villani(2008)]{Villani2008OptimalTO}
C{\'e}dric Villani.
\newblock Optimal transport: Old and new.
\newblock 2008.

\bibitem[Zhang and Chen(2022)]{zhang2022fast}
Qinsheng Zhang and Yongxin Chen.
\newblock Fast sampling of diffusion models with exponential integrator.
\newblock \emph{arXiv preprint arXiv:2204.13902}, 2022.

\bibitem[Zhang et~al.(2018)Zhang, Isola, Efros, Shechtman, and Wang]{zhang2018perceptual}
Richard Zhang, Phillip Isola, Alexei~A Efros, Eli Shechtman, and Oliver Wang.
\newblock The unreasonable effectiveness of deep features as a perceptual metric.
\newblock In \emph{CVPR}, 2018.

\bibitem[Zhu et~al.(2017)Zhu, Park, Isola, and Efros]{Zhu2017UnpairedIT}
Jun-Yan Zhu, Taesung Park, Phillip Isola, and Alexei~A. Efros.
\newblock Unpaired image-to-image translation using cycle-consistent adversarial networks.
\newblock \emph{2017 IEEE International Conference on Computer Vision (ICCV)}, pages 2242--2251, 2017.

\end{thebibliography}
}

\newpage
\begin{center}
    \Large\textbf{Appendix: Denoising Diffusion Bridge Models}
\end{center}
\appendix

\section{Proofs}\label{sec:proof}

\subsection{Marginal distribution}\label{sec:marginal-proof}

We note that for tractable transition kernels specified in \tabref{tab:bridge_ins}, we can derive the marginal distribution of $\rvx_t$ using Bayes' rule
$$p(\rvx_t\mid\rvx_0, \rvx_T) = \frac{p(\rvx_T\mid\rvx_t) p(\rvx_t\mid \rvx_0)}{p(\rvx_T\mid\rvx_0)}$$
We can directly derive this by looking at the resulting density function. First,
\begin{align}
    p(\rvx_t\mid \rvx_0) &= \frac{1}{\sqrt{2\pi} \sigma_t} \exp( - \frac{(\rvx_t - \alpha_t\rvx_0)^2}{2\sigma_t^2})\\
     p(\rvx_T\mid \rvx_t) &= \frac{1}{\sqrt{2\pi} \sqrt{\sigma_T^2 - \frac{\alpha_T^2}{\alpha_t^2}\sigma_t^2}} \exp( - \frac{(\frac{\alpha_T}{ \alpha_t}\rvx_t - \rvx_T)^2}{2(\sigma_T^2 - \frac{\alpha_T^2}{\alpha_t^2}\sigma_t^2)})\\ 
     &= \frac{1}{\sqrt{2\pi} \sqrt{\sigma_T^2 - \frac{\alpha_T^2}{\alpha_t^2}\sigma_t^2}} \exp( - \frac{ (\rvx_t - \frac{\alpha_t}{\alpha_T} \rvx_T)^2}{2 \sigma_t^2(\frac{\text{SNR}_t}{\text{SNR}_T} - 1)})\\
     p(\rvx_T\mid \rvx_0) &= \frac{1}{\sqrt{2\pi} \sigma_T} \exp( - \frac{(\rvx_T - \alpha_T\rvx_0)^2}{2\sigma_T^2})
\end{align}
and we refer readers to \citet{kingma2021variational} for details on $p(\rvx_s\mid \rvx_t)$ for any $s>t$. Then we know 
\begin{align}
    &p(\rvx_t\mid\rvx_0, \rvx_T)\\
    =& \frac{1}{\sqrt{2\pi}  \underbrace{\frac{\sigma_t}{\sigma_T} \sqrt{\sigma_T^2 - \frac{\alpha_T^2}{\alpha_t^2}\sigma_t^2}}_{\textstyle\hat{\sigma}_t}} \exp(- \frac{1}{2} \smash[b]{\underbrace{\Big[ \frac{(\rvx_t - \alpha_t\rvx_0)^2}{\sigma_t^2} + \frac{ (\rvx_t - \frac{\alpha_t}{\alpha_T} \rvx_T)^2}{ \sigma_t^2(\frac{\text{SNR}_t}{\text{SNR}_T} - 1)} - \frac{(\rvx_T - \alpha_T\rvx_0)^2}{\sigma_T^2} }_{\textstyle-\frac{\textstyle(\rvx_t - \hat{\mu}_t)^2}{\textstyle2 \hat{\sigma}_t^2 }} } \Big] )
\end{align}
where 
\begin{align}
    \hat{\sigma}_t^2 &= \sigma_t^2(1 - \frac{\text{SNR}_T}{\text{SNR}_t})\\
    \hat{\mu}_t &= \frac{\text{SNR}_T}{\text{SNR}_t} \frac{\alpha_t}{\alpha_T}\rvx_T + \alpha_t \rvx_0(1-\frac{\text{SNR}_T}{\text{SNR}_t})
\end{align}

\subsection{Denoising Bridge Score Matching}
\naiveobj*

\begin{proof}
We can explicitly write the objective as
\begin{equation}
    \resizebox{\hsize}{!}{$\int_{\rvx_t, \rvx_0,\rvx_T,t}q(\rvx_t\mid \rvx_0, \rvx_T)\pdata(\rvx_0,\rvx_T)  w(t) p(t)\Big[ \norm{\rvs_\theta(\rvx_t,\rvx_T, t) - \gradnd{\log q(\rvx_t\mid\rvx_0, \rvx_T)}{\rvx_t}}^2 \Big] d\rvx_t d\rvx_0d\rvx_Tdt $}
\end{equation}
    Since the objective is an $\gL_2$ loss and $p(t),w(t)$ are non-zero, its minimum can be derived as
    \begin{align}
       &\rvs^*(\rvx_t, \rvx_T, t) \nonumber \\
       =& \int_{\rvx_0,t}\frac{q(\rvx_t\mid \rvx_0, \rvx_T)\pdata(\rvx_0,\rvx_T) \cancel{w(t)p(t)}}{\int_{\rvx_0}p(\rvx_t\mid \rvx_0, \rvx_T)\pdata(\rvx_0,\rvx_T)\cancel{ w(t)p(t)} d\rvx_0 } \gradnd{\log q(\rvx_t\mid\rvx_0, \rvx_T)} d\rvx_0 dt \\
       =& \int_{\rvx_0}\frac{q(\rvx_t\mid \rvx_0, \rvx_T)\pdata(\rvx_0,\rvx_T) }{q(\rvx_t, \rvx_T)} \gradnd{\log q(\rvx_t\mid\rvx_0, \rvx_T)}{\rvx_t} d\rvx_0 \\
       =&\int_{\rvx_0}\frac{\cancel{q(\rvx_t\mid \rvx_0, \rvx_T)}\pdata(\rvx_0,\rvx_T) }{q(\rvx_t, \rvx_T)} \frac{\gradnd{q(\rvx_t\mid\rvx_0, \rvx_T)}{\rvx_t}}{\cancel{q(\rvx_t\mid\rvx_0, \rvx_T)}} d\rvx_0 \\
       =&\frac{\gradnd{\int_{\rvx_0} \pdata(\rvx_0,\rvx_T) q(\rvx_t\mid\rvx_0, \rvx_T)}{\rvx_t} d\rvx_0 }{q(\rvx_t ,\rvx_T)}\\
       =&\frac{\gradnd{q(\rvx_t, \rvx_T)}{\rvx_t}}{q(\rvx_t, \rvx_T)}\\
       =&\gradnd{\log q(\rvx_t\mid \rvx_T)}{\rvx_t}
    \end{align}
    Thus, minimizing the objective approximates the conditional score.
\end{proof}

\subsection{Probability Flow ODE of Diffusion Bridges}\label{sec:ode-proof}
\revbridge*
\begin{proof}
To find the time evolution of $q(\rvx_t \mid \rvx_T) = \int_{\rvx_0}p(\rvx_t\mid\rvx_0, \rvx_T)\pdata(\rvx_0\mid \rvx_T)$, we can first find the time evolution of  $p(\rvx_t\mid\rvx_0=x_0, \rvx_T=x_T)$ for fixed endpoints $x_0$ and $x_T$, which by Bayes' rule is
$$p(\rvx_t\mid \rvx_T=x_T,\rvx_0=x_0) = \frac{p(\rvx_T=x_T\mid\rvx_t)p(\rvx_t\mid\rvx_0=x_0)}{p(\rvx_T=x_T\mid\rvx_0=x_0)}$$
where $p(\rvx_t\mid\rvx_0)$ follows Kolmogorov forward equation
\begin{align}
    \parderiv{p(\rvx_t\mid\rvx_0=x_0)}{t} = - \nabla_{\rvx_t}\cdot \Big[\rvf(\rvx_t, t)p(\rvx_t\mid\rvx_0=x_0)\Big] + \frac{1}{2} g^2(t)\nabla_{\rvx_t}\cdot\gradnd{p(\rvx_t\mid\rvx_0=x_0)}{\rvx_t}    
\end{align}
and $p(\rvx_T=x_T\mid\rvx_t)$ follows Kolmogorov backward equation~\cite{szavits2015inequivalence} where
\begin{align}
    -\parderiv{p(\rvx_T=x_T\mid \rvx_t)}{t} = \rvf(\rvx_t, t)\cdot \gradnd{p(\rvx_T=x_T\mid\rvx_t)}{\rvx_t} + \frac{1}{2} g^2(t)\nabla_{\rvx_t}\cdot\gradnd{p(\rvx_T=x_T\mid \rvx_t)}{\rvx_t}    
\end{align}

The time derivative of $p(\rvx_t\mid\rvx_T=x_T,\rvx_0=x_0)$ thus follows
\begin{align}
&\parderiv{p(\rvx_t\mid\rvx_T=x_T,\rvx_0=x_0)}{t} \\
=& \parderiv{\frac{p(\rvx_T=x_T\mid\rvx_t)p(\rvx_t\mid\rvx_0=x_0)}{p(\rvx_T=x_T\mid\rvx_0=x_0)}}{t}\\
=&\underbrace{\frac{p(\rvx_t\mid\rvx_0=x_0)}{p(\rvx_T=x_T\mid\rvx_0=x_0)}\parderiv{p(\rvx_T=x_T\mid\rvx_t)}{t}}_{\circled{1}} + \underbrace{\frac{p(\rvx_T=x_T\mid\rvx_t)}{p(\rvx_T=x_T\mid\rvx_0=x_0)}\parderiv{p(\rvx_t\mid\rvx_0=x_0)}{t}}_{\circled{2}}
\end{align}
% For ease of notation, we here on omit the constant value $x$ and $y$ the variables $\rvx_T$ and $\rvx_0$ takes on and define for now that $p(\rvx_t\mid\rvx_T,\rvx_0)\defeq p(\rvx_t\mid\rvx_T=x_T,\rvx_0=x_0)$, $p(\rvx_T\mid\rvx_t)\defeq p(\rvx_T=x_T\mid\rvx_t)$
Further expanding the right-hand-side, we have
\begin{align*}
\begin{split}
    \circled{1}&=-\frac{p(\rvx_t\mid\rvx_0=x_0)}{p(\rvx_T=x_T\mid\rvx_0=x_0)} \Big(\rvf(\rvx_t, t)\cdot \gradnd{p(\rvx_T=x_T\mid\rvx_t)}{\rvx_t} + \frac{1}{2} g^2(t)\nabla_{\rvx_t}\cdot\gradnd{p(\rvx_T=x_T\mid \rvx_t)}{\rvx_t} \Big)
\end{split}
\end{align*}

\begin{align*}
\begin{split}
    \circled{2}&=\frac{p(\rvx_T=x_T\mid\rvx_t)}{p(\rvx_T=x_T\mid\rvx_0=x_0)}\Big(- \nabla_{\rvx_t}\cdot \Big[\rvf(\rvx_t, t)p(\rvx_t\mid\rvx_0=x_0)\Big] + \frac{1}{2} g^2(t)\nabla_{\rvx_t}\cdot\gradnd{p(\rvx_t\mid\rvx_0=x_0)}{\rvx_t} \Big)
\end{split}
\end{align*}
We can notice that the sum of the first terms of $\circled{1}$ and $\circled{2}$ is the result of a product rule, thus
\begin{equation}\label{eq:ode-reduce-1}
\begin{split}
    \circled{1} + \circled{2} &= -\nabla_{\rvx_t}\cdot\Big[\rvf(\rvx_t, t)p(\rvx_t\mid\rvx_T=x_T,\rvx_0=x_0)\Big] \\&
    +\frac{1}{2}g^2(t)\Bigg(\frac{p(\rvx_T=x_T\mid\rvx_t)}{p(\rvx_T=x_T\mid\rvx_0=x_0)} \nabla_{\rvx_t} \cdot\gradnd{p(\rvx_t\mid\rvx_0=x_0)}{\rvx_t}\\&
    - \frac{p(\rvx_t\mid\rvx_0=x_0)}{p(\rvx_T=x_T\mid\rvx_0=x_0)}\nabla_{\rvx_t}\cdot\gradnd{p(\rvx_T=x_T\mid\rvx_t)}{\rvx_t}\Bigg)
\end{split}
\end{equation}
We now focus on reducing the terms in the last bracket. For clarity, we similarly number the two terms inside the bracket such that
\begin{align}
    \circled{3} &= \frac{p(\rvx_T=x_T\mid\rvx_t)}{p(\rvx_T=x_T\mid\rvx_0=x_0)} \nabla_{\rvx_t} \cdot\gradnd{p(\rvx_t\mid\rvx_0=x_0)}{\rvx_t}\\
    \circled{4} &= \frac{p(\rvx_t\mid\rvx_0=x_0)}{p(\rvx_T=x_T\mid\rvx_0=x_0)}\nabla_{\rvx_t}\cdot\gradnd{p(\rvx_T=x_T\mid\rvx_t)}{\rvx_t}
\end{align}
Now we can complete these terms to be results of product rule by adding and subtracting the following term
\begin{align}
&\frac{\gradnd{p(\rvx_T=x_T\mid\rvx_t)}{\rvx_t}\cdot\gradnd{p(\rvx_t\mid\rvx_0=x_0)}{\rvx_t}}{p(\rvx_T=x_T\mid\rvx_0=x_0)}    \\
&=\underbrace{\frac{\gradnd{p(\rvx_t\mid\rvx_0=x_0)}{\rvx_t}}{p(\rvx_T=x_T\mid\rvx_0=x_0)}\cdot\Big[p(\rvx_T=x_T\mid\rvx_t)\gradnd{\log p(\rvx_T=x_T\mid\rvx_t)}{\rvx_t}\Big]}_{\circled{5}}\\
&=\underbrace{\frac{\gradnd{p(\rvx_T=x_T\mid\rvx_t)}{\rvx_t}}{p(\rvx_T=x_T\mid\rvx_0=x_0)}\cdot\Big[p(\rvx_t\mid\rvx_0=x_0)\gradnd{\log p(\rvx_t\mid\rvx_0=x_0)}{\rvx_t}\Big]}_{\circled{6}}
\end{align}
which takes 2 equivalent forms $\circled{5}$ and $\circled{6}$. Now we can write \equref{eq:ode-reduce-1} as
\begin{align}
    \circled{1}+\circled{2} &= -\nabla_{\rvx_t}\cdot\Big[\rvf(\rvx_t, t)p(\rvx_t\mid\rvx_T=x_T,\rvx_0=x_0)\Big] \\&
    +\frac{1}{2}g^2(t)\Bigg(\circled{3} + \circled{4} + \circled{5} + \circled{6}\Bigg) - g^2(t)\Bigg(\circled{4} + \circled{5}\Bigg)
\end{align}
We can notice that
\begin{align}
    \circled{3}+\circled{6} &= \nabla_{\rvx_t}\cdot\Bigg(p(\rvx_t\mid\rvx_T=x_T,\rvx_0=x_0)\gradnd{\log p(\rvx_t\mid\rvx_0=x_0)}{\rvx_t} \Bigg)\\
    \circled{4}+\circled{5} &= \nabla_{\rvx_t}\cdot\Bigg(p(\rvx_t\mid\rvx_T=x_T,\rvx_0=x_0)\gradnd{\log p(\rvx_T=x_T\mid\rvx_t)}{\rvx_t} \Bigg)
\end{align}
and using Bayes' rule,
\begin{align}
    \gradnd{\log p(\rvx_t\mid\rvx_T=x_T,\rvx_0=x_0)}{\rvx_t} = \gradnd{\log p(\rvx_T=x_T\mid\rvx_t)}{\rvx_t} + \gradnd{\log p(\rvx_t\mid\rvx_0=x_0)}{\rvx_t}
\end{align}
we have
\begin{align}
     \circled{3}+\circled{4}+\circled{5}+\circled{6} &= \nabla_{\rvx_t}\cdot\Bigg(p(\rvx_t\mid\rvx_T=x_T,\rvx_0=x_0)\gradnd{\log p(\rvx_t\mid\rvx_T=x_T,\rvx_0=x_0)}{\rvx_t} \Bigg)\\
     &=\nabla_{\rvx_t}\cdot\gradnd{p(\rvx_t\mid\rvx_T=x_T,\rvx_0=x_0)}{\rvx_t}
\end{align}
Therefore,
\begin{equation}
    \begin{aligned}
    &\parderiv{p(\rvx_t\mid\rvx_T=x_T,\rvx_0=x_0)}{t} \\=&  -\nabla_{\rvx_t}\cdot\Bigg[\Big(\rvf(\rvx_t, t)+g^2(t)\nabla_{\rvx_t}\log p(\rvx_T=x_T\mid\rvx_t)\Big)p(\rvx_t\mid\rvx_T=x_T,\rvx_0=x_0)\Bigg] \\
    &+\frac{1}{2}g^2(t)\nabla_{\rvx_t}\cdot\gradnd{p(\rvx_t\mid\rvx_T=x_T,\rvx_0=x_0)}{\rvx_t}
\end{aligned}
\end{equation}
which is a Fokker-Planck equation for a (forward) SDE with the modified drift term $$\rvf(\rvx_t, t)+g^2(t)\nabla_{\rvx_t}\log p(\rvx_T=x_T\mid\rvx_t)$$
To find the time derivative for $q(\rvx_t \mid \rvx_T) = \int_{\rvx_0}p(\rvx_t\mid\rvx_0, \rvx_T)\pdata(\rvx_0\mid \rvx_T)$, we can simply marginalize out $\rvx_0$ with distribution $\pdata(\rvx_0\mid\rvx_T)$ in the resulting Fokker-Planck, which can be achieved due to linearity of expectation with respect to $\rvx_0$. That is,
\begin{equation}
    \begin{aligned}
    &\E_{\rvx_0\sim \pdata(\rvx_0\mid \rvx_T=x_T)}\Big[\parderiv{p(\rvx_t\mid\rvx_T=x_T,\rvx_0)}{t}\Big]\\
    =& -\nabla_{\rvx_t}\cdot\Bigg[\Big(\rvf(\rvx_t, t)+g^2(t)\nabla_{\rvx_t}\log p(\rvx_T=x_T\mid\rvx_t)\Big)\E_{\rvx_0\sim \pdata(\rvx_0\mid \rvx_T=x_T)}\Big[p(\rvx_t\mid\rvx_T=x_T,\rvx_0)\Big]\Bigg] \\
    & +\frac{1}{2}g^2(t)\nabla_{\rvx_t}\cdot\gradnd{\E_{\rvx_0\sim \pdata(\rvx_0\mid \rvx_T=x_T)}\Big[p(\rvx_t\mid\rvx_T=x_T,\rvx_0)\Big]}{\rvx_t}
\end{aligned}
\end{equation}

Since for $t\in [0,T-c]$ for some $c>0$, Doob's h-function is well-defined, and $p(x_t\mid x_0, x_T)$ is smooth, and we can take the expectation inside the equations. Additionally, the drift adjustment $\nabla_{\rvx_t}\log p(\rvx_T=x_T\mid\rvx_t)$ does not depend on $\rvx_0$ the expectation is simply over $p(\rvx_t\mid\rvx_T=x_T,\rvx_0)$ expectation and by definition LHS is $q(\rvx_t\mid\rvx_T=x_T)$,
\begin{equation}
\begin{aligned}
    \parderiv{q(\rvx_t\mid\rvx_T=x_T)}{t} =& -\nabla_{\rvx_t}\cdot\Bigg[\Big(\rvf(\rvx_t, t)+g^2(t)\nabla_{\rvx_t}\log p(\rvx_T=x_T\mid\rvx_t)\Big)q(\rvx_t\mid\rvx_T=x_T)\Bigg] \\
    &+\frac{1}{2}g^2(t)\nabla_{\rvx_t}\cdot\gradnd{q(\rvx_t\mid\rvx_T=x_T)}{\rvx_t}
\end{aligned}    
\end{equation}
This characterizes a reverse SDE specified in \thmref{thm:revbridge}.

We can further use conversion trick in \citet{song2020score} to convert this into a continuity equation without any diffusion term where
\begin{align}
    \parderiv{q(\rvx_t\mid\rvx_T=x_T)}{t} = \nabla_{\rvx_t}\cdot\Big[\Tilde{\rvf}(\rvx_t, t)q(\rvx_t\mid\rvx_T=x_T)\Big]
\end{align}
where \begin{align}
\Tilde{\rvf}(\rvx_t, t) &= \rvf(\rvx_t, t)+g^2(t)\nabla_{\rvx_t}\log p(\rvx_T=x_T\mid\rvx_t) -\frac{1}{2}g^2(t)\nabla_{\rvx_t} \log q(\rvx_t\mid\rvx_T=x_T)\\
&= \rvf(\rvx_t, t)-g^2(t)\Big(\frac{1}{2}\nabla_{\rvx_t}\log q(\rvx_t \mid\rvx_T = x_T) - \nabla_{\rvx_t}\log p(\rvx_T=x_T\mid\rvx_t)\Big)
\end{align}
% where the second equality uses Bayes' rule and the third equality simply splits the second term into two halves and uses one half to cancel out the third term. We thus arrive at our result. 
\end{proof}

\subsection{Special Cases of Denoising Diffusion Bridges}\label{sec:special-case-proof}

\mypara{Unconditional diffusion models.} We first give a general intuition that the marginal distribution of $\rvx_t$ sampling from the bridge is the same as sampling marginally from $p(\rvx_t \mid \rvx_0)$ for a diffusion transition kernel $p(\cdot)$. We can see this by observing 
\begin{align}
    \rvx_t = \frac{\text{SNR}_T}{\text{SNR}_t} \frac{\alpha_t}{\alpha_T}\rvx_T + \alpha_t \rvx_0(1-\frac{\text{SNR}_T}{\text{SNR}_t}) +  \sigma_t^2(1 - \frac{\text{SNR}_T}{\text{SNR}_t}) \rvepsilon_1
\end{align}
where $\rvepsilon_1\sim \gN(\vzero, \mI)$. And since we assume $\rvx_T\sim\gN(\alpha_T\rvx_0, \sigma_T^2\mI), $we rewrite the above equation as
\begin{align}
     \rvx_t &= \frac{\text{SNR}_T}{\text{SNR}_t} \frac{\alpha_t}{\alpha_T}(\alpha_T\rvx_0 + \sigma_T\rvepsilon_2)+ \alpha_t \rvx_0(1-\frac{\text{SNR}_T}{\text{SNR}_t}) +  \sigma_t\sqrt{(1 - \frac{\text{SNR}_T}{\text{SNR}_t})} \rvepsilon_1\\
        &= \frac{\text{SNR}_T}{\text{SNR}_t} \alpha_t\rvx_0 + \frac{\text{SNR}_T}{\text{SNR}_t}\frac{\alpha_t}{\alpha_T}\sigma_T\rvepsilon_2 + \alpha_t \rvx_0(1-\frac{\text{SNR}_T}{\text{SNR}_t}) +  \sigma_t\sqrt{(1 - \frac{\text{SNR}_T}{\text{SNR}_t})} \rvepsilon_1\\
        &=\alpha_t\rvx_0 + \sigma_t\rvepsilon
\end{align}
where $\rvepsilon\sim \gN(\vzero, \mI)$ and the last equality is due to the fact that the addition of two Gaussian with variances $\sigma_1^2$, $\sigma_2^2$ is another Gaussian with variance $\sigma_1^2 + \sigma_2^2$.

Formally, to show that it is a special case, we first observe that the score matching objective allows our network to approximate $\nabla_{\rvx_t}\log p(\rvx_t\mid \rvx_T)$ which is the conditional score of the diffusion transition kernel. Then we will show that the Fokker-Planck equation reduces to that of a diffusion when marginalizing out dependency on $\rvx_T$.

From proof of \thmref{thm:revbridge}, we know that the Fokker-Planck equation for $p(\rvx_t\mid \rvx_T)$ follows
\begin{align}
    \parderiv{p(\rvx_t\mid\rvx_T=x_T)}{t} =& -\nabla_{\rvx_t}\cdot\Bigg[\Big(\rvf(\rvx_t, t)+g^2(t)\nabla_{\rvx_t}\log p(\rvx_T=x_T\mid\rvx_t)\Big)p(\rvx_t\mid\rvx_T=x_T)\Bigg] \\
    &+\frac{1}{2}g^2(t)\nabla_{\rvx_t}\cdot\gradnd{p(\rvx_t\mid\rvx_T=x_T)}{\rvx_t}
\end{align}
Here we note that we use $p(\rvx_t\mid\rvx_T)$ because we are considering a diffusion process as a special case of a general $q(\rvx_t\mid \rvx_T)$ introduced in \thmref{thm:naiveobj}. We can marginalize out $\rvx_T$ such that
\begin{align}
     &\parderiv{\E_{\rvx_T\sim p(\rvx_T)}\Big[p(\rvx_t\mid\rvx_T)\Big]}{t} \nonumber \\
     =& \E_{\rvx_T\sim p(\rvx_T)}\Bigg[-\nabla_{\rvx_t}\cdot\Bigg[\Big(\rvf(\rvx_t, t)+g^2(t)\nabla_{\rvx_t}\log p(\rvx_T\mid\rvx_t)\Big)p(\rvx_t\mid\rvx_T)\Bigg] \Bigg] \nonumber\\
    &+\frac{1}{2}g^2(t)\nabla_{\rvx_t}\cdot\gradnd{\E_{\rvx_T\sim p(\rvx_T)}\Big[p(\rvx_t\mid\rvx_T)\Big]}{\rvx_t}
\end{align}
and so
\begin{align}
    \parderiv{p(\rvx_t)}{t} =& -\nabla_{\rvx_t}\cdot\Big(\rvf(\rvx_t, t) p(\rvx_t)\Big) - g^2(t)\nabla_{\rvx_t}\cdot \E_{\rvx_T\sim p(\rvx_T)}\Big[p(\rvx_t\mid\rvx_T) \nabla_{\rvx_t}\log p(\rvx_T\mid\rvx_t)\Big]\nonumber \\
    &+\frac{1}{2}g^2(t)\nabla_{\rvx_t}\cdot\gradnd{p(\rvx_t)}{\rvx_t}
\end{align}
and the second term can be reduced by writing the expectation explicitly as
\begin{align}
     &\E_{\rvx_T\sim p(\rvx_T)}\Big[p(\rvx_t\mid\rvx_T) \nabla_{\rvx_t}\log p(\rvx_T\mid\rvx_t)\Big]\nonumber\\
     =& \int_{\rvx_T} p(\rvx_T)p(\rvx_t\mid\rvx_T) \nabla_{\rvx_t}\log p(\rvx_T\mid\rvx_t) d\rvx_T\\
     =& p(\rvx_t)\int_{\rvx_T} p(\rvx_T\mid\rvx_t) \nabla_{\rvx_t}\log p(\rvx_T\mid\rvx_t) d\rvx_T\\
     =& p(\rvx_t)\int_{\rvx_T} \cancel{p(\rvx_T\mid\rvx_t)} \frac{\nabla_{\rvx_t}p(\rvx_T\mid\rvx_t)}{\cancel{p(\rvx_T\mid\rvx_t)}} d\rvx_T\\
     =& p(\rvx_t)\nabla_{\rvx_t}\int_{\rvx_T}p(\rvx_T\mid\rvx_t)d\rvx_T\\
     =& \vzero
\end{align}
Therefore, the resulting probability flow ODE is
\begin{align}
    \parderiv{p(\rvx_t)}{t} =& -\nabla_{\rvx_t}\cdot\Big(\rvf(\rvx_t, t) p(\rvx_t)\Big) + \frac{1}{2}g^2(t)\nabla_{\rvx_t}\cdot\gradnd{p(\rvx_t)}{\rvx_t}
\end{align}
which is that of a regular diffusion. Therefore, by setting data distribution $\pdata(\rvx_0, \rvx_T)$ to be $p(\rvx_T\mid \rvx_0)\pdata(\rvx_0)$ we recover unconditional diffusion models.

\mypara{OT-Flow Matching and Rectified Flow.} As proposed, we use a VE schedule such that $\rvf(\rvx_t, t) = \vzero$ and $\sigma_t^2 = c^2t$ for some constant $c\in [0,1]$. Then the probability flow ODE conditioned on $\rvx_0, \rvx_T$ becomes
\begin{align}
    d\rvx_t = - c^2 \Big[\frac{1}{2} \gradnd{\log q(\rvx_t\mid\rvx_0, \rvx_T)}{\rvx_t} - \log p(\rvx_T\mid \rvx_0)\Big] dt
\end{align}
Specifically, the drift term $D = - c^2 \Big[\frac{1}{2} \gradnd{\log q(\rvx_t\mid\rvx_0, \rvx_T)}{\rvx_t} - \log p(\rvx_T\mid \rvx_0)\Big]$ becomes
\begin{align}
    D = -\frac{1}{2}c^2\Bigg[ -\frac{\rvepsilon}{c\sqrt{t(1 - \frac{t}{T})}} + 2 \frac{(\frac{t}{T} \rvx_T + (1- \frac{t}{T}) \rvx_0 + c\sqrt{t(1 - \frac{t}{T})} \rvepsilon - \rvx_T)}{c^2(T - t)} \Bigg]
\end{align}
where $\rvx_t = \frac{t}{T} \rvx_T + (1- \frac{t}{T}) \rvx_0 + c\sqrt{t(1 - \frac{t}{T})} \rvepsilon $.
And we can rearrange the terms to be
\begin{align}
    D &= \Bigg[ -  \frac{(\frac{t}{T} \rvx_T + (1- \frac{t}{T}) \rvx_0  - \rvx_T)}{(T - t)} \Bigg] + \gO(c)\\
    &=  \Bigg[ -  \frac{( (1- \frac{t}{T}) \rvx_0  - (1- \frac{t}{T}) \rvx_T)}{(T - t)} \Bigg] + \gO(c)\\
    & = \Bigg[ \frac{(\rvx_T  - \rvx_0)}{T} \Bigg] + \gO(c)
\end{align}
And by taking $c \to 0$, we have $\lim_{c\to 0} D = \rvx_1 - \rvx_0$ for $T = 1$.
Therefore the network learns to match this drift term in the noiseless limit of denoising diffusion bridge in OT-Flow Matching and Rectified Flow case.

We next note that the original score-matching loss is no longer valid as bridge noise $\hat{\sigma}_t \rightarrow 0$ causes exploding magnitude of bridge score $\gradnd{\log q(\rvx_t | \rvx_0, \rvx_T)}{\rvx_t}$. We can then resort to matching against $\lim_{c\to 0} D $ altogether. 
 
 One additional caveat is that our framework as presented needs to take in $\rvx_T$ as an additional condition. To handle this, we note the generalized parameterization can be used to define $s_\theta(\rvx_t, \rvx_T, t) = c_\text{skip1}(t) \rvx_t + c_\text{skip2}(t) \rvx_T + c_\text{out}(t) V_\theta(\rvx_t, t)$ where $V_\theta(\rvx_t, t)$ is our actual network. We then set $c_\text{skip1}(t) = c_\text{skip2}(t) = 0$ and $c_\text{out}(t) = 1$ and uses loss $\E_{\rvx_t, t}\Bigg[\norm{s_\theta(\rvx_t, \rvx_T, t) - (\rvx_T - \rvx_0)}^2\Bigg] = \E_{\rvx_t, t}\Bigg[\norm{V_\theta(\rvx_t, \rvx_T, t) - (\rvx_T - \rvx_0)}^2\Bigg] $, which is the case of OT-Flow-Matching and Rectified Flow.

\subsection{Generalized Parameterization}\label{sec:edm-proof}

We now derive the EDM scaling functions from first principle, as suggested by \citep{karras2022elucidating}. 

Let  $a_t = \alpha_t/\alpha_T* \text{SNR}_T/\text{SNR}_t$, $b_t = \alpha_t (1 -  \text{SNR}_T/\text{SNR}_t)$, $c_t = \sigma_t^2 (1- \text{SNR}_T/\text{SNR}_t)$. First, we expand the pred-$\rvx$ objective as
\[
\E_{\rvx_t, \rvx_0, \rvx_T, t}\Big[\Tilde{w}(t) \norm{ c_\text{skip}(t) \rvx_t + c_\text{out}F_\theta(c_\text{in}(t) \rvx_t, c_\text{noise}(t)) - \rvx_0}^2\Big]
\]
where $\rvx_t = a_t\rvx_T +  b_t\rvx_0 + \sqrt{c_t}\rvepsilon$ for $\rvepsilon\sim\gN(\vzero, \mI)$. 
To derive $c_{\text{in}}(t)$, we set the variance of the resulting input $c_\text{in}(t)\rvx_t$ to be 1, where
\begin{align}
    &c_\text{in}^2(t)\Big(a_t^2\sigma_T^2 +  b_t^2\sigma_0^2 + 2a_tb_t\sigma_{0T} + c_t\Big) = 1\\
    \implies &c_\text{in}(t) = \frac{1}{\sqrt{a_t^2\sigma_T^2 +  b_t^2\sigma_0^2 + 2a_tb_t\sigma_{0T} + c_t}}
\end{align}
For simplicity we denote the neural network as $F_\theta$, and the inner square loss can be expanded to be
\begin{align}
    &\Tilde{w}(t)\norm{ c_\text{skip}(t) \Big(a_t\rvx_T +  b_t\rvx_0 + \sqrt{c_t}\rvepsilon\Big)+ c_\text{out}F_\theta - \rvx_0}^2\\
    =& \Tilde{w}(t) c_\text{out}^2(t) \norm{F_\theta -\frac{1}{c_\text{out}(t)} \Bigg( \Big[1 - c_\text{skip}(t) b_t\Big]\rvx_0 - c_\text{skip}(t)\Big[a_t\rvx_T + \sqrt{c_t}\rvepsilon\Big] \Bigg) }^2
\end{align}
And we want the prediction target to have variance 1, thus
\begin{equation}
\frac{1}{c_\text{out}^2(t)} \Bigg( \Big[1 - c_\text{skip}(t) b_t\Big]^2\sigma_0^2 + c_\text{skip}(t)^2\Big[a_t\sigma_T^2 + c_t\Big]  - 2\Big[1 - c_\text{skip}(t) b_t\Big]c_\text{skip}(t)a_t\sigma_{0T} \Bigg) = 1
\end{equation}
and 
\begin{equation}
c_\text{out}^2(t) = \Big[1 - c_\text{skip}(t) b_t\Big]^2\sigma_0^2 + c_\text{skip}(t)^2\Big[a_t\sigma_T^2 + c_t\Big] - 2\Big[1 - c_\text{skip}(t) b_t\Big]c_\text{skip}(t)a_t\sigma_{0T}
\end{equation}
Following reasoning in \citet{karras2022elucidating}, we minimize $c_\text{out}(t)^2$ \wrt $c_\text{skip}(t)$ by taking derivative and set to 0, which is
\begin{equation}
-2(1 - c_\text{skip}(t) b_t)b_t\sigma_0^2 + 2c_\text{skip}(t)(a_t^2\sigma_T^2 + c_t) - 2(1 - 2c_\text{skip}(t)b_t) a_t\sigma_{0T} = 0
\end{equation}
and this implies
\begin{align}
    c_\text{skip}(t) &= \frac{b_t\sigma_0^2 + a_t\sigma_{0T}}{a_t^2\sigma_T^2 +  b_t^2\sigma_0^2 + 2a_tb_t\sigma_{0T} + c_t}\\
    &= \Big(b_t\sigma_0^2 + a_t\sigma_{0T}\Big) * c_\text{in}(t)^2
\end{align}
And
\begin{align}
    c_\text{out}^2(t) &= \sigma_0^2 - 2c_\text{skip}(t) b_t\sigma_0^2 + \Big(b_t\sigma_0^2 + a_t\sigma_{0T} \Big)c_\text{skip}(t) - 2 c_\text{skip}(t) a_t\sigma_{0T}\\
    &= \sigma_0^2 - \Big( b_t\sigma_0^2 + a_t\sigma_{0T} \Big)c_\text{skip}(t) \\
    &= \frac{a_t^2(\sigma_0^2\sigma_T^2 - \sigma_{0T}^2) + \sigma_0^2 c_t}{a_t^2\sigma_T^2 +  b_t^2\sigma_0^2 + 2a_tb_t\sigma_{0T} + c_t}\\
    \implies & c_\text{out}(t) = \sqrt{a_t^2(\sigma_0^2\sigma_T^2 - \sigma_{0T}^2) + \sigma_0^2 c_t} * c_\text{in}(t) 
\end{align}
Finally, $\Tilde{w}(t)  c_\text{out}(t)^2 (t)=1 \implies \Tilde{w}(t)  =1 / c_\text{out}(t)^2  $, and for time, we simply reuse that proposed in \citet{karras2022elucidating} as no significant change in time's distribution. 

\mypara{EDM \citep{karras2022elucidating} as a special case.} In the case of unconditional diffusion models, we have $\rvx_T = \rvx_0 + T\rvepsilon$, so $\sigma_T^2 = \sigma_0^2 + T^2$ and $\sigma_{0T} = \sigma_0^2$. Additionally, $a_t = t^2 / T^2$, $b_t = (1 - t^2/T^2)$, $c_t = t^2 (1-t^2/T^2)$. Substituting in these into the coefficients, we have
\begin{align}
    c_\text{in}(t) &= \frac{1}{\sqrt{\frac{t^4}{T^4}( \sigma_0^2 + T^2)+  (1 - \frac{t^2}{T^2})^2\sigma_0^2 + 2\frac{t^2}{T^2}(1 - \frac{t^2}{T^2})\sigma_0^2  + t^2 (1 - \frac{t^2}{T^2})}}\\
    &=  \frac{1}{\sqrt{\cancel{\frac{t^4}{T^4}\sigma_0^2} +\cancel{\frac{t^4}{T^2}} +  (1 - \frac{t^2}{T^2})^2\sigma_0^2 + 2\frac{t^2}{T^2}\sigma_0^2 -\cancel{2}\frac{t^4}{T^4}\sigma_0^2 + t^2 - \cancel{\frac{t^4}{T^2}}  }}\\
    &= \frac{1}{\sqrt{\sigma_0^2 +t^2}}\\
    c_\text{skip}(t) &= \frac{(1 - \frac{t^2}{T^2})\sigma_0^2 + \frac{t^2}{T^2}\sigma_0^2}{\sigma_0^2 +t^2}\\
    &= \frac{\sigma_0^2}{\sigma_0^2 +t^2}\\
    c_\text{out}(t) &= \sqrt{\frac{t^4}{T^4}(\sigma_0^2(\sigma_0^2 + T^2) - \sigma_{0}^4) + \sigma_0^2 t^2 (1 - \frac{t^2}{T^2})} * c_\text{in}(t)\\
    &= \sqrt{\frac{t^4}{T^4}(\sigma_0^4 + \sigma_0^2 T^2 - \sigma_0^4)  + \sigma_0^2 t^2 (1 - \frac{t^2}{T^2})} * c_\text{in}(t)\\
    &=\sqrt{\frac{t^4}{T^2}\sigma_0^2   + \sigma_0^2 t^2 - \sigma_0^2  \frac{t^4}{T^2}} * c_\text{in}(t)\\
    & = \frac{ \sigma_0 t}{\sqrt{\sigma_0^2 +t^2}}
\end{align}
And $\Tilde{w}(t) = 1 / c_\text{out}^2(t) = 
(\sigma_0^2 +t^2) / (\sigma_0^2 t^2) = 1 / t^2 + 1 / \sigma_0^2$.

\subsection{Sampler Discretization}\label{sec:discretization}

EDM introduces Heun sampler, which discretizes the sampling steps into $t_0 < t_1 \dots < t_N$ where
\begin{align}
    t_{i>0} = \Big(T^{\frac{1}{\rho}} + \frac{N - i}{N-1}(t_{\text{min}}^{\frac{1}{\rho}} - T^{\frac{1}{\rho}})\Big)^\rho \quad \text{and} \quad t_0 = 0
\end{align}
and $\rho = 7$ is a default choice. It then integrates over the probability flow ODE path with second-order Heun steps for each such discretization step. We reuse this discretization for all our experiments.

\section{Experiment Details}\label{sec:exp-details}

\mypara{Architecture.} For unconditional generation, architectures are reused from \citet{karras2022elucidating} for both CIFAR-10 and FFHQ-64$\times$64. For pixel-space translation, we use ADM~\citep{dhariwal2021diffusion} architecture for both 64$\times$64 and 256$\times$256 resolutions. For latent-space translation, which reduces to 32$\times$32 resolution in the latent space, we use ADM~\citep{dhariwal2021diffusion} architecture for 64$\times$64 resolution but change the channel dimensions from 192 to 256 and reduce the number of residual blocks from 3 to 2, and we fix everything else to be same as that for 64$\times$64 resolution. We use 0.1 dropout for all models. Conditioning is done via concatenation at the input level.

\mypara{VE and VP bridge parameterization.} There are many schedules we can choose for both types of bridges. For all our experiments, VE bridges follow $\sigma_t = t$ and $\alpha_t = 1$ and VP bridges follow a time-invariant drift $\rvf(\rvx_t, t) = - 0.5 \beta_0 \rvx_t$. This is a special case of the linear noise schedule considered in \citet{song2020score} with $\rvf(\rvx_t, t) = -0.5 t (\beta_1 - \beta_0) - 0.5 \beta_0$. We simply choose $\beta_1 = \beta_0$ because this results in a bridge that has symmetric noise levels \wrt time. We observe that linearly increasing drift causes the max noise to shift towards a higher $t$ and the noise decreases faster to 0 at $t=T$ than for a symmetric bridge. This makes the learning process more difficult and degrades performance.

\mypara{Training.} We use AdamW optimizer with 0.0001 learning rate and no weight decay. The batch size is 256 for all image size less than 256 and training is done on 4 NVIDIA A100 40G. For 256$\times$256 resolution, the batch size is 4 accumulated 4 times such that the effective batch size is 64, trained on 4 NVIDIA A100 40G. The training is terminated at 500K iterations. During training, for image-to-image translation, we set $\sigma_0 = \sigma_T = 0.5$, $\sigma_{0T} = \sigma_0^2 / 2$, and for unconditional generation, we set $\sigma_0 = 0.5$, $\sigma_T = \sqrt{\sigma_0^2 + T^2}$ and $\sigma_{0T} = \sigma_0^2$. We use random flipping as our data augmentation for image-to-image translation and reuse augmentation from \citet{karras2022elucidating} for generation.

\mypara{Baselines.} All baselines are trained using the same architecture as ours for each experiment. For SDEdit, we use pretrained EDM model on $\rvx_0$ and conduct image-to-image translation by first noising $\rvx_T$ and denoising using the pretrained model. We reuse the noise schedule proposed by \citep{karras2022elucidating} and for reasonable generation while retaining global structure of the image conditions, we noise $\rvx_T$ using the noise variance indexed at 1/3 of EDM noise schedule and denoise starting from this noised image for the remaining 1/3 of total of $N$ steps. For DDIB, we train two separate unconditional models starting for $\rvx_0$ and $\rvx_T$ separately and perform translation by reversing DDIM starting from $\rvx_T$ and generating using DDIM for $\rvx_0$. We reuse the original baseline code for all baselines while. 

\mypara{Sampling.} For all experiments we evaluate models on a low-step regime, \ie the same number of sampling steps. For all experiments, we set guidance scale $w=0.5$ and for image translation and unconditional generation, we use euler step ratio ratio $s=0.33$ and $s=0$ respectively. In case of $s=0$, no Euler step is done. With these settings, we set $N=18$, or $\text{NFE}=53$, for 32$\times$32 resolution image translation, and for all other resolutions, we use $N=40$, or $\text{NFE}=118$, for image translation. For unconditional generation, $N=18 \implies \text{NFE}=36$ for CIFAR-10 and $N=40 \implies \text{NFE}=79$ for FFHQ-64$\times$64. FID and IS scores are calculated using the entire training set for all datasets for image translation tasks. They are calculated using 50K samples for unconditional generation tasks. 

\subsection{Additional Results.}\label{sec:more-res}

\mypara{Latent space translation.} Many real world applications of diffusion models rely on the existence of a latent space~\citep{rombach2022high}, which significantly relieves computational burden in practice. We thus also investigate whether \model can naturally adapt to latent space translation tasks. We conduct our experiment on Day$\to$Night~\citep{isola2017image} which is chosen because we observe that the autoencoder presented in~\citet{rombach2022high} with factor 8 resolution reduction can faithfully reconstruct the dataset in high quality. We reuse the metrics and baselines from before, and we evaluate all models with $N=18$ using the EDM sampler default for $32\times 32$ images, which is the size for the latent space. To isolate the evaluation for latent space translation, we use the reconstructed ground-truths by the autoencoder as our pseudo ground-truths against which the decoded samples are compared for all metrics. Quantitative results are presented in \tabref{tab:pix-i2i}.

\begin{table}[h]
    \centering
    \resizebox{0.5\linewidth}{!}{
\begin{tabular}{lcccc}
\toprule
  &  \multicolumn{4}{c}{Day$\to$Night-256$\times$256}  \\
\cmidrule{1-5}
& FID $\downarrow$ & IS $\uparrow$ & LPIPS $\downarrow$& MSE $\downarrow$ \\
\midrule
Pix2Pix &157.1 &2.48 & 0.622 & 0.148 \\
\midrule
SDEdit  & 151.1 & 2.93 &0.582 &0.412 \\
\midrule
% Cond-EDM & \textbf{0.815}  & 3.56 & 0.231 &   &  & & & &  \\
DDIB & 226.9  &  2.11 & 0.789 &0.824 \\
\midrule
Rectified Flow & \textbf{ 12.38} & 3.90 & 0.366  & \textbf{0.129} \\
\midrule
I$^{2}$SB & 15.56 & \textbf{4.03} & \textbf{0.363} & 0.412 \\
\midrule
% \midrule
% \textbf{\model}-Det & 14.37 &  &  & 58.47 & 3.70 & 0.436 & & &  \\
\midrule
\textbf{\model} & 27.63 & 3.92 & 0.549 & 0.145 \\
\bottomrule
\end{tabular}
}
\caption{Evaluation of latent-space image-to-image translation.}
    \label{tab:d2n}
\end{table}

We notice that our model performs less well compared to Rectified Flow and I$^2$SB but comes second in IS and MSE. We hypothesize that the predict-$\rvx$ parameterization is optimized for pixel-space generation and latent space contains structures that is more difficult for this parameterization to learn. Nevertheless, qualitative results from our model is high-quality, as shown in \figref{fig:d2n}.

\mypara{Additional visualization}

% \begin{figure}[h]
%     \centering
%     \includegraphics[width=0.95\linewidth]{figures/e2h_1.png}
%     \caption{Additional Edges$\to$Handbags results.}
% \end{figure}
\begin{figure}
    \centering
    \includegraphics[width=0.95\linewidth]{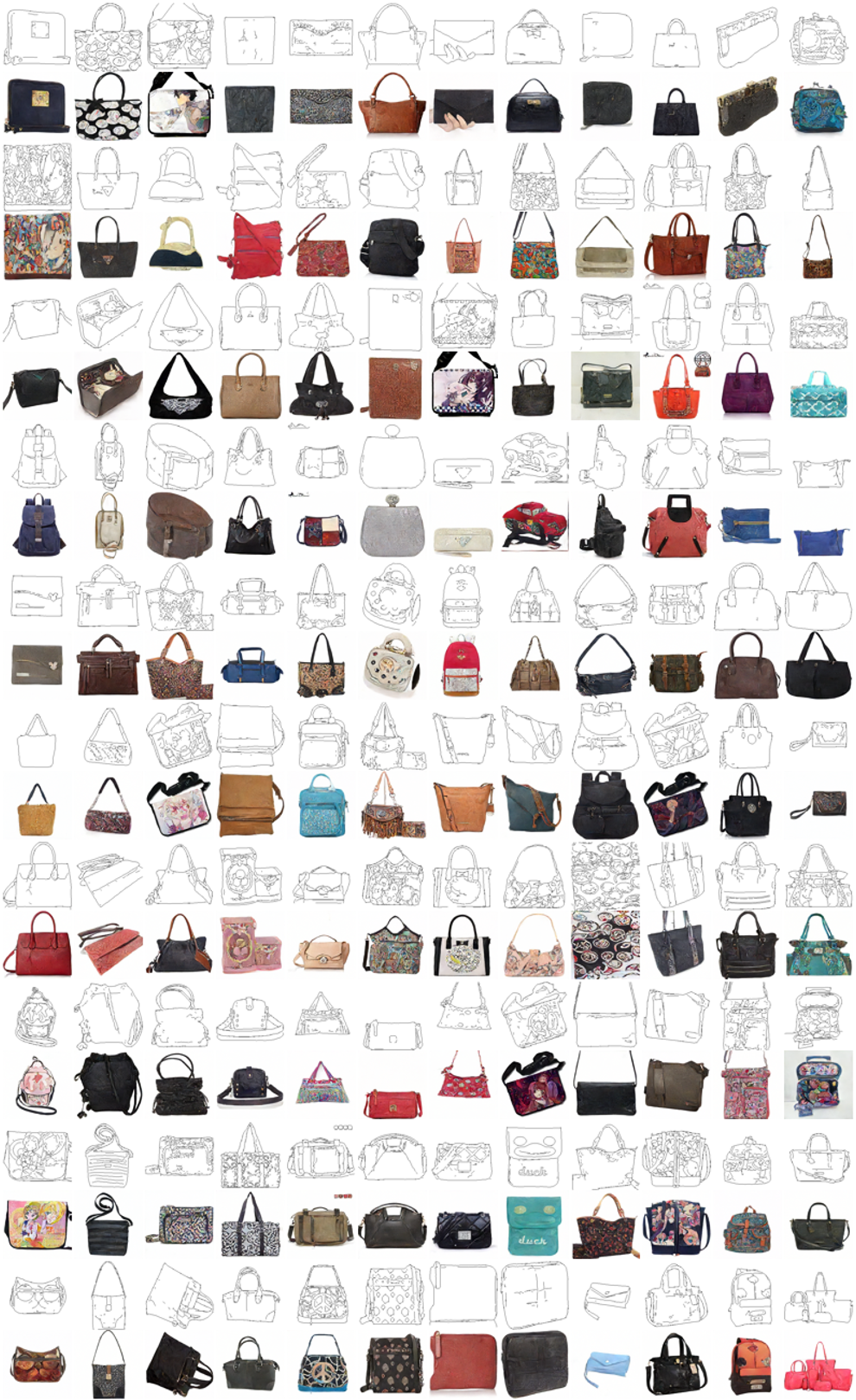}
    \caption{Additional Edges$\to$Handbags results.}
\end{figure}

\begin{figure}[h]
    \centering
    \includegraphics[width=0.95\linewidth]{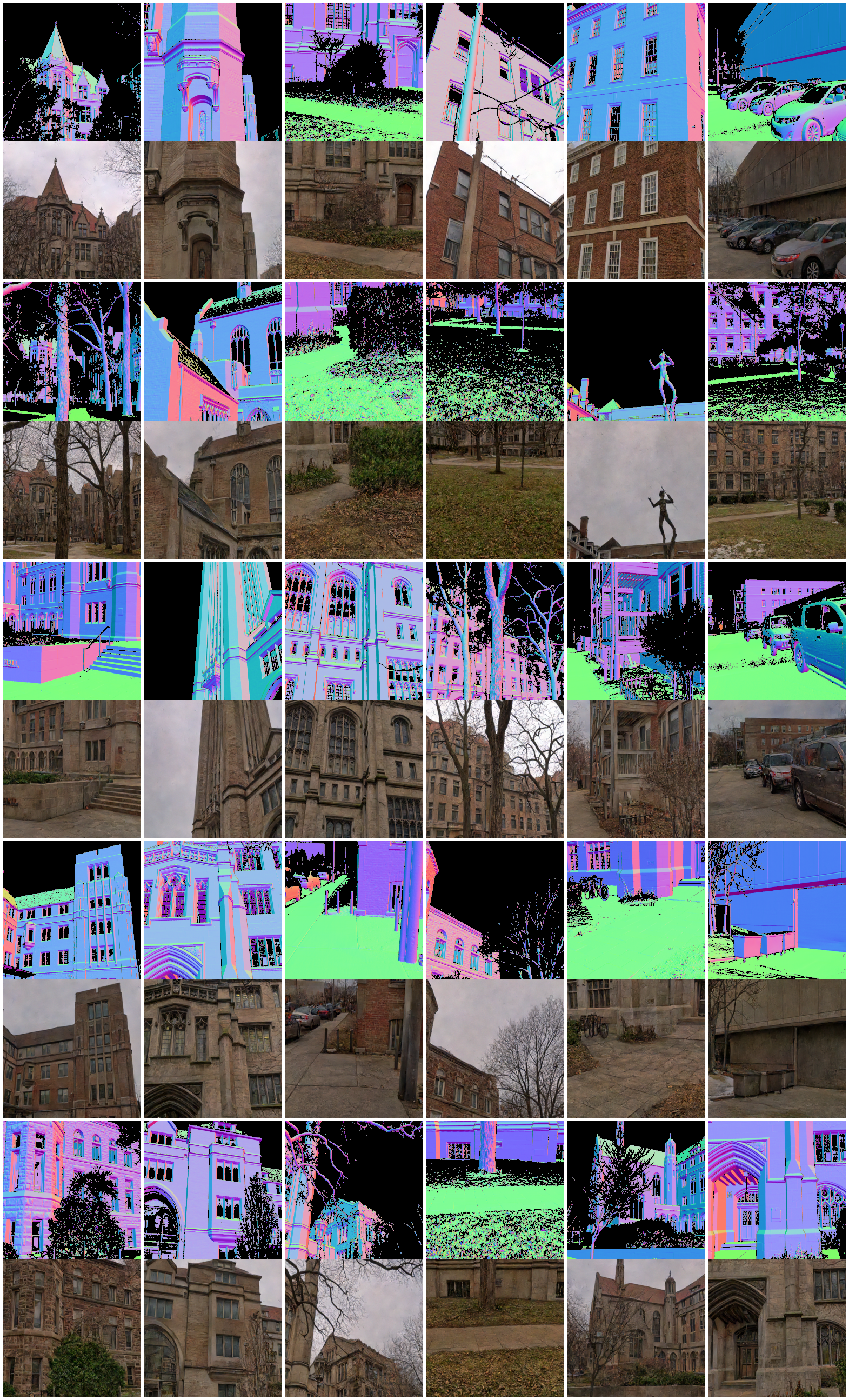}
    \caption{Additional DIODE results.}
\end{figure}

\begin{figure}
    \centering
    \includegraphics[width=0.95\linewidth]{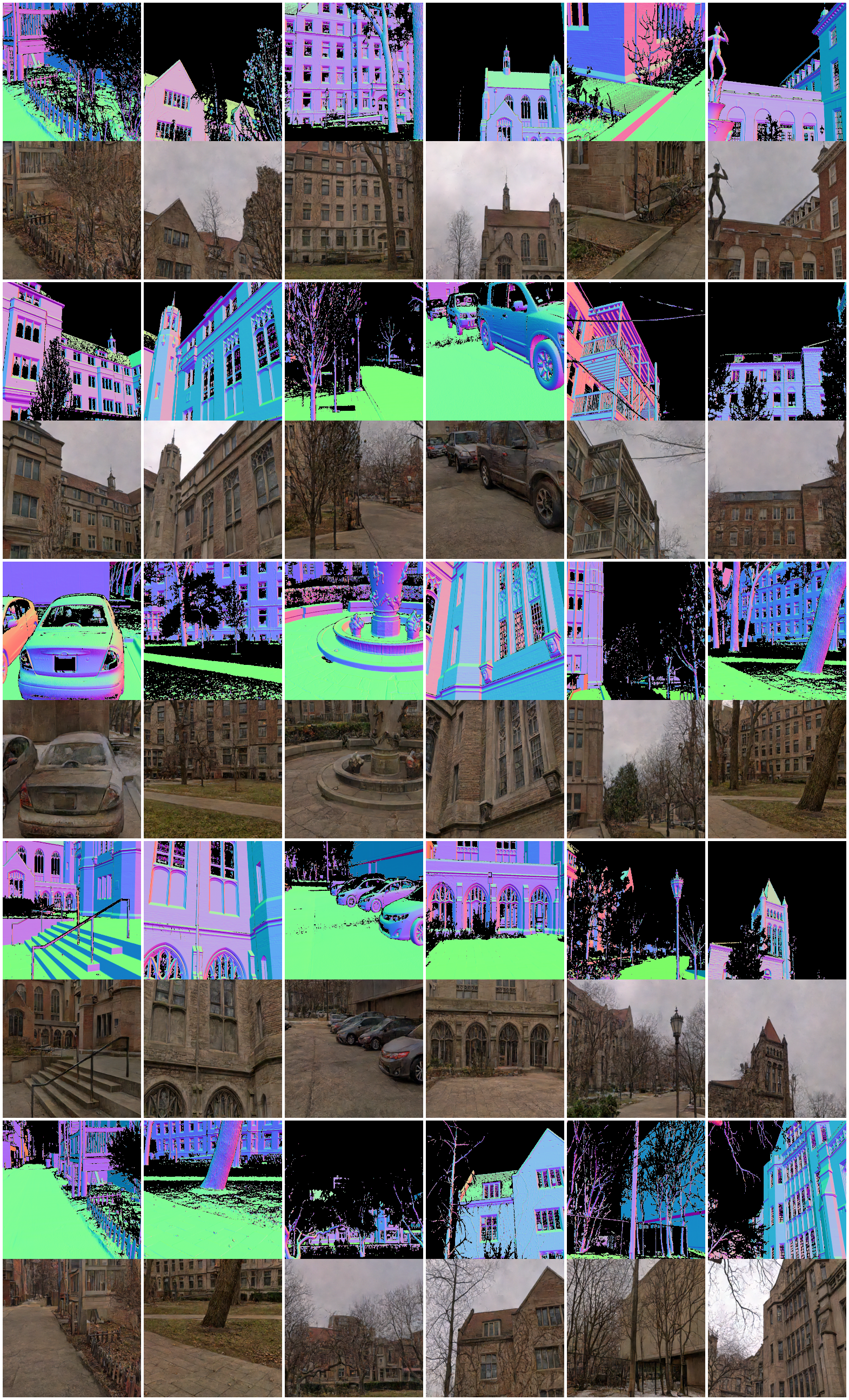}
    \caption{Additional DIODE results.}
\end{figure}

\begin{figure}[t]
    \centering
    \includegraphics[width=\linewidth]{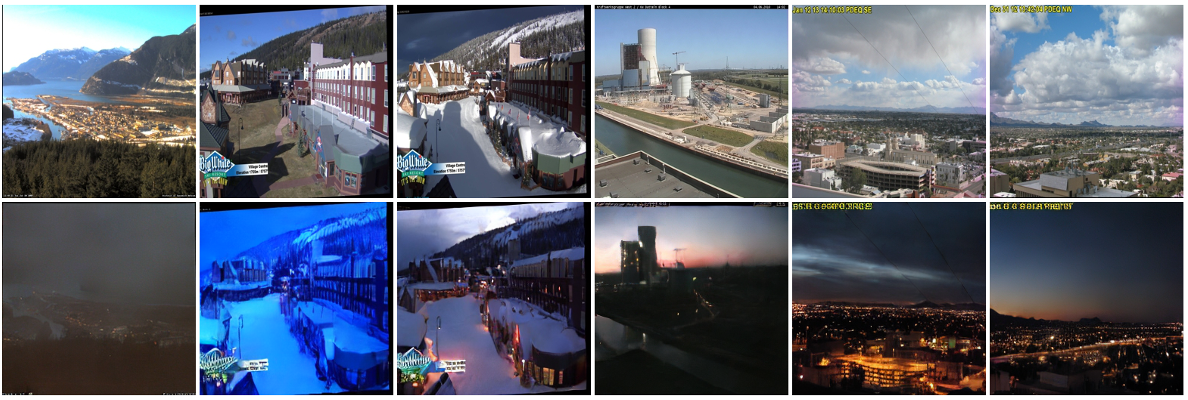}
    \caption{Visualization for Day$\to$Night translation. Top: day images. Bottom: night translations.}
    \label{fig:d2n}
\end{figure}

\end{document}